%% file: main-arxiv-bbwi.tex
\documentclass[12pt]{article}
\usepackage{fullpage}

\date{}

\title{\bfseries PAC Battling Bandits in the Plackett-Luce Model}

\author{
Aadirupa Saha\thanks{Indian Institute of Science, Bangalore, India. {\tt aadirupa@iisc.ac.in}}, \and Aditya Gopalan \thanks{Indian Institute of Science, Bangalore, India. {\tt aditya@iisc.ac.in} }
}

\usepackage{microtype}
\usepackage{cancel}
\usepackage{graphicx}
\usepackage{booktabs} 
\usepackage{amssymb}
\usepackage{color}
\usepackage{amsmath}
\usepackage{amsthm}
\usepackage{thmtools}
\usepackage{thm-restate}
\usepackage{algorithm}
\usepackage{algorithmic}

\usepackage{natbib}
\usepackage{hyperref}
\usepackage{nameref}

\allowdisplaybreaks

\newtheorem{thm}{Theorem}
\newtheorem{lem}[thm]{Lemma}

\newtheorem{cor}[thm]{Corollary}
\newtheorem{defn}[thm]{Definition}

\newtheorem{rem}{Remark}

\newcommand{\R}{{\mathbb R}}

\renewcommand{\P}{{\mathbf P}}

\newcommand{\E}{{\mathbf E}}

\newcommand{\1}{{\mathbf 1}}

\newcommand{\cO}{{\mathcal O}}
\newcommand{\cA}{{\mathcal A}}

\newcommand{\cB}{{\mathcal B}}

\newcommand{\cC}{{\mathcal C}}

\newcommand{\cE}{{\mathcal E}}

\newcommand{\cF}{{\mathcal F}}
\newcommand{\cG}{{\mathcal G}}

\newcommand{\cN}{{\mathcal N}}

\newcommand{\cD}{{\mathcal D}}

\newcommand{\cR}{{\mathcal R}}

\newcommand{\X}{{\mathcal X}}

\newcommand{\hp}{{\hat p}}
\newcommand{\p}{{\mathbf p}}
\newcommand{\q}{{\mathbf q}}

\newcommand{\tp}{{\tilde p}}
\def \algtrc{{\it Trace-the-Best}}
\def \algdiv{{\it Divide-and-Battle}}
\def \algmed{{\it Halving-Battle}}
\def \algupdt{{\it Rank-Breaking}}

\newcommand{\btheta}{\boldsymbol \theta}
\newcommand{\bSigma}{\boldsymbol \Sigma}

\newcommand{\bnu}{{\boldsymbol \nu}}

\newcommand{\bsigma}{\boldsymbol \sigma}

\begin{document}

\maketitle

\input{abstract.tex}

\input{introduction.tex}

\input{prelims.tex}

\input{prob_setup.tex}

\input{res_wi_lb.tex}

\input{res_wi_algo.tex}

\input{res_fr_lb.tex}

\input{res_fr_algo.tex}


\input{conclusion.tex}

\newpage

\bibliographystyle{plainnat}
\bibliography{pl-battling-bandits}

\newpage

\input{appendix.tex}

\end{document}

%% file: abstract.tex

\begin{abstract}
We introduce the probably approximately correct (PAC) \emph{Battling-Bandit} problem with the Plackett-Luce (PL) subset choice model--an online learning framework where at each trial the learner chooses a subset of $k$ arms from a fixed set of $n$ arms, and subsequently observes a stochastic feedback indicating preference information of the items in the chosen subset, e.g., the most preferred item or ranking of the top $m$ most preferred items etc. The objective is to identify a near-best item in the underlying PL model with high confidence. 
This generalizes the well-studied PAC \emph{Dueling-Bandit} problem over $n$ arms, which aims to recover the \emph{best-arm} from pairwise preference information, and is known to require $O(\frac{n}{\epsilon^2} \ln \frac{1}{\delta})$ sample complexity \citep{Busa_pl,Busa_top}. We study the sample complexity of this problem under various feedback models: (1) Winner of the subset (WI), and (2) Ranking of top-$m$ items (TR) for $2\le m \le k$. We show, surprisingly, that with winner information (WI) feedback over subsets of size $2 \leq k \leq n$, the best achievable sample complexity is still $O\left( \frac{n}{\epsilon^2} \ln \frac{1}{\delta}\right)$, independent of $k$, and the same as that in the Dueling Bandit setting ($k=2$). 
For the more general top-$m$ ranking (TR) feedback model, we show a significantly smaller lower bound on sample complexity of $\Omega\bigg( \frac{n}{m\epsilon^2} \ln \frac{1}{\delta}\bigg)$, which suggests a multiplicative reduction by a factor ${m}$ owing to the additional information revealed from preferences among $m$ items instead of just $1$. We also propose two algorithms for the PAC problem with the TR feedback model with optimal (upto logarithmic factors) sample complexity guarantees, establishing the increase in statistical efficiency from exploiting rank-ordered feedback. 
\end{abstract}

%% file: introduction.tex
\section{Introduction}
\label{sec:intro}
The {dueling bandit} problem has recently gained attention in the machine learning community \citep{Yue+12,Ailon+14,Zoghi+14RUCB,Busa_pl}. This is a variant of the multi-armed bandit problem \citep{Auer+02} in which the learner needs to learn an `best arm' from pairwise comparisons between arms. %
In this work, we consider a natural generalization of the dueling bandit problem where the learner can adaptively select a subset of $k$ arms ($k \geq 2$) in each round, and observe relative preferences in the subset following a Plackett-Luce (PL) feedback model \citep{Marden_book}, with the objective of learning the `best arm'. We call this the \emph{battling bandit} problem with the Plackett-Luce model. 

The battling bandit decision framework \citep{SG18,ChenSoda+18} models several application domains where it is possible to elicit feedback about preferred options from among a general set of offered options, instead of being able to compare only two options at a time as in the dueling setup. Furthermore, the phenomenon of competition -- that an option's utility or attractiveness is often assessed relative to that of other items in the offering -- is captured effectively by a subset-dependent stochastic choice model such as Plackett-Luce. Common examples of learning settings with such feedback include recommendation systems and search engines, medical interviews, tutoring systems--any applications where relative preferences from a chosen pool of options are revealed. 


We consider a natural probably approximately correct (PAC) learning problem in the battling bandit setting: %
Output an $\epsilon$-approximate best item (with respect to its Plackett-Luce parameter) with probability at least $(1-\delta)$, while keeping the total number of adaptive exploration rounds small. We term this the $(\epsilon,\delta)$-{PAC} objective of searching for an approximate winner or top-$1$ item.


Our primary interest lies in understanding how the subset size $k$ influences the sample complexity of achieving $(\epsilon,\delta)$-{PAC} objective in subset choice models for various feedback information structures, e.g., winner information ({WI}), which returns only a single winner of the chosen subset, or the more general top ranking ({TR}) information structure, where an ordered tuple of $m$ `most-preferred' items is observed. More precisely, we ask: Does being able to play size-$k$ subsets help learn optimal items faster than in the dueling setting ($k=2$)?  How does this depend on the subset size $k$, and on the feedback information structure? How much, if any, does rank-ordered feedback accelerate the rate of learning, compared to only observing winner feedback? %
%
This paper takes a step towards resolving such questions within the context of the Plackett-Luce choice model. Among the contributions of this paper are: 

\begin{enumerate}
\item We frame a \emph{PAC version of Battling Bandits} with $n$ arms -- a natural generalization of the PAC-Dueling-Bandits problem \citep{Busa_pl} -- 
with the objective of finding an $\epsilon$-approximate best item with probability at least $1-\delta$ with minimum possible sample complexity, termed as the $(\epsilon,\delta)$-{PAC} objective (Section \ref{sec:obj}).
 
\item We consider learning with winner information {(WI)} feedback, where the learner can play a subsets $S_t \subseteq [n]$ of exactly $|S_t| = k$ distinct elements at each round $t$, following which a winner of $S_t$ is observed according to an underlying, unknown, Plackett-Luce model. We show an information-theoretic lower bound on sample complexity for $(\epsilon,\delta)$-{PAC} of $\Omega\bigg( \frac{n}{\epsilon^2} \ln \frac{1}{\delta}\bigg)$ rounds (Section \ref{sec:lb_wi}), which is of the same order as that for the dueling bandit ($k=2$) \citep{BTM}. This implies that, despite the increased flexibility of playing sets of potentially large size $k$, with just winner information feedback, one cannot hope for a faster rate of learning than in the case of pairwise selections. Intuitively, competition among a large number ($k$) of elements vying for the top spot at each time exactly offsets the potential gain that being able to test more alternatives together brings. 
On the achievable side, we design two algorithms (Section \ref{sec:algo_wi}) for the $(\epsilon,\delta)$-{PAC} objective, and derive sample complexity guarantees which are optimal within a logarithmic factor of the lower bound derived earlier. 
 When the learner is allowed to play subsets of sizes $1,2, \ldots$ upto $k$, which is a slightly more flexible setting than above, we design a median elimination-based algorithm with order-optimal $O\left(\frac{n}{\epsilon^2}\ln \frac{1}{\delta}\right)$ sample complexity which, when specialized to $k=2$, improves upon existing sample complexity bounds for PAC-dueling bandit algorithms, e.g. \cite{BTM,Busa_pl} under the PL model (Section. \ref{sec:reswi_anyk}).

\item We next study the $(\epsilon,\delta)$-PAC problem in a more general top-ranking {(TR)} feedback model where the learner gets to observe the ranking of top $m$ items drawn from the Plackett-Luce distribution, $2 \le m \le k$ (Section \ref{sec:feed_mod}), departing from prior work. For $m = 1$, the setting simply boils down to WI feedback model. In this case, we are able to prove a sample complexity lower bound of $\Omega\bigg( \frac{n}{m\epsilon^2} \ln \frac{1}{\delta}\bigg)$ 
 (Theorem \ref{thm:lb_pacpl_rnk}), which suggests that with top-$m$ ranking ({TR}) feedback, it may be possible to aggregate information $m$ times faster than with just winner information feedback. We further present {two} algorithms (Section \ref{sec:alg_fr}) for this problem which, are shown to enjoy optimal (upto logarithmic factors) sample complexity guarantees. This formally shows that the $m$-fold increase in statistical efficiency by exploiting richer information contained in top-$m$ ranking feedback is, in fact, algorithmically achievable. 

\item From an algorithmic point of view, we elucidate how the structure of the Plackett-Luce choice model, such as its independent of irrelevant attributes (IIA) property, play a crucial role in allowing the development of parameter estimates, together with tight confidence sets, which form the basis for our learning algorithms. It is indeed by leveraging this property (Lemma \ref{lem:pl_simulator}) that we afford to maintain consistent pairwise preferences of the items by applying the concept of \emph{Rank Breaking} to subsetwise preference data. This significantly alleviates the combinatorial explosion that could otherwise result if one were to keep more general subset-wise estimates.



\end{enumerate}


\textbf{Related Work:}
Statistical parameter estimation in Plackett-Luce models has been studied in detail in the offline batch (non-adaptive) setting \citep{SueIcml+15,KhetanOh16,SueIcml+17}. 
%

In the online setting, there is a fairly mature body of work concerned with PAC best-arm (or top-$\ell$ arm) identification in the classical multi-armed bandit \citep{Even+06,Audibert+10,Kalyanakrishnan+12,Karnin+13,pmlr-v35-jamieson14}, where absolute utility information is assumed to be revealed upon playing a single arm or item. Though most work on dueling bandits has focused on the regret minimization goal \citep{Zoghi+14RUCB,Ramamohan+16}, there have been recent developments on the PAC objective for different pairwise preference models, such as those satisfying stochastic triangle inequalities and strong stochastic transitivity \citep{BTM}, general utility-based preference models \citep{SAVAGE}, the Plackett-Luce model \citep{Busa_pl}, the Mallows model \citep{Busa_mallows}, etc. Recent work in the PAC setting focuses on learning objectives other than identifying the single (near) best arm, e.g. recovering a few of the top arms \citep{Busa_top,MohajerIcml+17,ChenSoda+17}, or the true ranking of the items \citep{Busa_aaai,falahatgar_nips}. 

The work which is perhaps closest in spirit to ours is that of \citet{ChenSoda+18}, which addresses the problem of learning the top-$\ell$ items in Plackett-Luce battling bandits. Even when specialized to $\ell = 1$ (as we consider here), however, this work differs in several important aspects from what we attempt. \citet{ChenSoda+18} develop algorithms for the probably {\em exactly} correct objective (recovering a near-optimal arm is not favored), and, consequently, show instance-dependent sample complexity bounds, whereas we allow a tolerance of $\epsilon$ in defining best arms, which is often natural in practice \cite{Busa_pl,BTM}. As a result, we bring out the dependence of the sample complexity on the specified tolerance level $\epsilon$, rather than on purely instance-dependent measures of hardness. Also, their work considers only winner information (WI) feedback from the subsets chosen, whereas we consider, for the first time, general top-$m$ ranking information feedback. 


A related battling-type bandit setting has been studied as the MNL-bandits assortment optimization problem by \citet{Agrawal+16}, although it takes prices of items into account when defining their utilities. 
As a result, their work optimizes for a subset with highest expected revenue (price), whereas we search for a best item (Condorcet winner). and the two settings are in general incomparable. 


%% file: prelims.tex
\section{Preliminaries}
\label{sec:prelims}
{\bf Notation.} We denote by $[n]$ the set $\{1,2,...,n\}$. For any subset $S \subseteq [n]$, let $|S|$ denote the cardinality of $S$. 
When there is no confusion about the context, we often represent (an unordered) subset $S$ as a vector, or ordered subset, $S$ of size $|S|$ (according to, say, a fixed global ordering of all the items $[n]$). In this case, $S(i)$ denotes the item (member) at the $i$th position in subset $S$.   
$\bSigma_S = \{\sigma \mid \sigma$  is a permutation over items of $ S\}$, where for any permutation $\sigma \in \Sigma_{S}$, $\sigma(i)$ denotes the element at the $i$-{th} position in $\sigma, i \in [|S|]$.
$\1(\varphi)$ is generically used to denote an indicator variable that takes the value $1$ if the predicate $\varphi$ is true, and $0$ otherwise. 
$x \vee y$ denotes the maximum of $x$ and $y$, and $Pr(A)$ is used to denote the probability of event $A$, in a probability space that is clear from the context.

\subsection{Discrete Choice Models and Plackett-Luce (PL)}
\label{sec:RUM}  

A discrete choice model specifies the relative preferences of two or more discrete alternatives in a given set. %
%
A widely studied class of discrete choice models is the class of {\it Random Utility Models} (RUMs), which assume a ground-truth utility score $\theta_{i} \in \R$ for each alternative $i \in [n]$, and assign a conditional distribution $\cD_i(\cdot|\theta_{i})$ for scoring item $i$. To model a winning alternative given any set $S \subseteq [n]$, one first draws a random utility score $X_{i} \sim \cD_i(\cdot|\theta_{i})$ for each alternative in $S$, and selects an item with the highest random score. 

%

One widely used RUM is the {\it Multinomial-Logit (MNL)} or {\it Plackett-Luce model (PL)}, where the $\cD_i$s are taken to be independent Gumbel distributions with location parameters $\theta'_i$ and scale parameter $1$ \citep{Az+12}, which result to probability densities $\cD_i(x_{i}|\theta'_i) = e^{-(x_j - \theta'_j)}e^{-e^{-(x_j - \theta'_j)}}$, $\theta'_i \in R, ~ \forall i \in [n]$. Moreover assuming $\theta'_i = \ln \theta_i$, $\theta_i > 0 ~\forall i \in [n]$, in this case the probability that an alternative $i$ emerges as the winner in the set $S \ni i$ becomes proportional to its parameter value:
\begin{align}
\label{eq:prob_PL}
Pr(i|S) = \frac{{\theta_i}}{\sum_{j \in S}{\theta_j}}.
\end{align}
We will henceforth refer the above choice model as PL model with parameters $\btheta = (\theta_1, \ldots, \theta_n)$.
 Clearly the above model induces a {total ordering} on the arm set $[n]$: If $p_{ij} = P(i \succ j) = Pr(i|\{i,j\}) = \frac{{\theta_i}}{{\theta_i}+{\theta_j}}$ denotes the pairwise probability of item $i$ being preferred over item $j$, then $p_{ij} \ge \frac{1}{2}$ if and only if $\theta_i \ge \theta_j$, or in other words if $p_{ij} \ge \frac{1}{2}$ and $p_{jk} \ge \frac{1}{2}$ then $p_{ik} \ge \frac{1}{2}$, $\forall i,j,k \in [n]$ \citep{Ramamohan+16}.

Other families of discrete choice models can be obtained by imposing different probability distributions over the utility scores $X_i$, e.g. if $(X_1,\ldots X_n) \sim \cN(\btheta,\boldsymbol \Lambda)$ are jointly normal with mean $\btheta = (\theta_1,\ldots \theta_n)$ and covariance $\boldsymbol \Lambda \in \R^{n \times n}$, then the corresponding RUM-based choice model reduces to the {\it Multinomial Probit (MNP)}. Unlike MNL, though, the choice probabilities $Pr(i|S)$ for the MNP model do not admit a closed-form expression \citep{RUMegs}.
 
\subsection{Independence of Irrelevant Alternatives}
\label{sec:iia}
A choice model $Pr$ is said to possess the {\it Independence of Irrelevant Alternatives (IIA)} property if the ratio of probabilities of choosing any two items, say $i_1$ and $i_2$ from within any choice set $S \ni {i_1,i_2}$ is independent of a third alternative $j$ present in $S$ \citep{IIA-relevance16}. More specifically,
$
\frac{Pr(i_1|S_1)}{Pr(i_2|S_1)} = \frac{Pr(i_1|S_2)}{Pr(i_2|S_2)} \text{ for any two distinct subsets } S_1,S_2 \subseteq [n]
$ 
that contain $i_1$ and $i_2$. One example of such a choice model is Plackett-Luce. 
\begin{rem}
IIA turns out to be very valuable in estimating the parameters of a PL model, with high confidence, via \emph{Rank-Breaking} -- the idea of extracting pairwise comparisons from (partial) rankings and applying estimators on the obtained pairs, treating each comparison independently. Although this technique has previously been used in batch (offline) PL estimation \citep{KhetanOh16}, we show that it can be used in online problems for the first time. We crucially exploit this property of the PL model in the algorithms we design (Algorithms \ref{alg:trace_bb}-\ref{alg:half_bb}), and in establishing their correctness and sample complexity guarantees.
\end{rem}



\vspace*{-5pt}

\begin{restatable}[Deviations of pairwise win-probability estimates for PL model]
{lem}{plsimulator}
\label{lem:pl_simulator}
\hspace*{-7pt} Consider a Plackett-Luce choice model with parameters $\btheta = (\theta_1,\theta_2, \ldots, \theta_n)$ (see Eqn. \eqref{eq:prob_PL}), and fix two distinct items $i,j \in [n]$. Let $S_1, \ldots, S_T$ be a sequence of (possibly random) subsets of $[n]$ of size at least $2$, where $T$ is a positive integer, and $i_1, \ldots, i_T$ a sequence of random items with each $i_t \in S_t$, $1 \leq t \leq T$, such that for each $1 \leq t \leq T$, (a) $S_t$ depends only on $S_1, \ldots, S_{t-1}$, and (b) $i_t$ is distributed as the Plackett-Luce winner of the subset $S_t$, given $S_1, i_1, \ldots, S_{t-1}, i_{t-1}$ and $S_t$, and (c) $\forall t: \{i,j\} \subseteq S_t$ with probability $1$. Let $n_i(T) = \sum_{t=1}^T \1(i_t = i)$ and $n_{ij}(T) = \sum_{t=1}^T \1(\{i_t \in \{i,j\}\})$. Then, for any positive integer $v$, and $\eta \in (0,1)$,
\[
Pr\left( \frac{n_i(T)}{n_{ij}(T)} \hspace*{-2pt} - \hspace*{-2pt} \frac{\theta_i}{\theta_i + \theta_j} \ge \eta, n_{ij}(T) \geq v \right) \vee \, Pr\left( \frac{n_i(T)}{n_{ij}(T)} \hspace*{-2pt} - \hspace*{-2pt} \frac{\theta_i}{\theta_i + \theta_j} \le -\eta, n_{ij}(T) \geq v \right) \leq e^{-2v\eta^2}. 
\]
\end{restatable}



\begin{proof}\textbf{(sketch)}.
The proof uses a novel coupling argument to work in an equivalent probability space for the PL model with respect to the item pair $i,j$, as follows. Let $Z_1, Z_2, \ldots$ be a sequence of iid Bernoulli random variables with success parameter $\theta_i/(\theta_i + \theta_j)$. A counter $C$ is first initialized to $0$. At each time $t$, given $S_1, i_1, \ldots, S_{t-1}, i_{t-1}$ and $S_t$,  an independent coin is tossed with probability of heads $(\theta_i + \theta_j)/\sum_{k \in S_t} \theta_k$. If the coin lands tails, then $i_t$ is drawn as an independent sample from the Plackett-Luce distribution over $S_t \setminus \{i,j\}$, else, the counter is incremented by $1$, and $i_t$ is returned as $i$ if $Z_C = 1$ or $j$ if $Z_C = 0$. 
This construction yields the correct joint distribution for the sequence $i_1, S_1, \ldots, i_T, S_T$, because of the IIA property of the PL model: 
	\[ Pr(i_t = i | i_t \in \{i,j\}, S_t) = \frac{Pr(i_t = i | S_t)}{Pr(i_t \in \{i,j\} | S_t)} = \frac{\theta_i/\sum_{k \in S_t} \theta_k}{(\theta_i + \theta_j)/\sum_{k \in S_t} \theta_k} = \frac{\theta_i}{\theta_i + \theta_j}. \]
The proof now follows by applying Hoeffding's inequality on prefixes of the sequence $Z_1,Z_2, \ldots$.
\end{proof}

  

%% file: prob_setup.tex
\vspace*{-17pt}

\section{Problem Setup}
\label{sec:prb_setup}

We consider the PAC version of the sequential decision-making problem of finding the {best item} in a set of $n$ items by making subset-wise comparisons. %
Formally, the learner is given a finite set $[n]$ of $n > 2$ arms. At each decision round $t = 1, 2, \ldots$, the learner selects a subset $S_t \subseteq [n]$ of $k$ 
distinct items, and receives (stochastic) feedback depending on (a) the chosen subset $S_t$, and (b) a Plackett-Luce (PL) choice model with parameters $\btheta = (\theta_1,\theta_2,\ldots, \theta_n)$ a priori unknown to the learner. The nature of the feedback can be of several types as described in Section \ref{sec:feed_mod}. Without loss of generality, we will henceforth assume $\theta_i \in [0,1], \,\forall i \in [n]$, since the PL choice probabilities are positive scale-invariant by  \eqref{eq:prob_PL}. We also let $\theta_1 > \theta_i \, \forall i \in [n]\setminus\{1\}$ for ease of  exposition\footnote{We naturally assume that this knowledge is not known to the learning algorithm, and note that extension to the case where several items have the same highest parameter value is easily accomplished.}. We call this decision-making model, parameterized by a PL instance $\btheta = (\theta_1,\theta_2,\ldots, \theta_n)$ and a playable subset size $k$, as {Battling Bandits (BB) with the Plackett-Luce (PL)}, or \emph{BB-PL} in short.
We define a \emph{best item} to be one with the highest score parameter: $i^* \in \underset{i \in [n]}{\text{argmax}}~\theta_i$. Under the assumptions above, $i^* = 1$ uniquely. Note that here we have $p_{1i} = P(1\succ i) > \frac{1}{2}$, $\forall i \in [n]\setminus\{1\}$, so item $1$ is the \emph{Condorcet Winner} \citep{Ramamohan+16} of the PL model.


\subsection{Feedback models}
\label{sec:feed_mod}
By feedback model, we mean the information received (from the `environment') once the learner plays a subset $S \subseteq [n]$ of $k$ items. We define three  types of feedback in the PL battling model:

\begin{itemize}

\item \textbf{Winner of the selected subset (WI):} 
The environment returns a single item $I \in S$, drawn independently from the probability distribution
$
\label{eq:prob_win}
Pr(I = i|S) = \frac{{\theta_i}}{\sum_{j \in S} \theta_j} ~~\forall i \in S.
$


\item \textbf{Full ranking selected subset of items (FR):} The environment returns a full ranking $\bsigma \in \bSigma_{S}$, drawn from the probability distribution
$
\label{eq:prob_rnk1}
Pr(\bsigma = \sigma|S) = \prod_{i = 1}^{|S|}\frac{{\theta_{\sigma(i)}}}{\sum_{j = i}^{|S|}\theta_{\sigma(j)}}, \; \sigma \in \bSigma_S.
$ 
In fact, this is equivalent to picking $\bsigma(1)$ according to the winner (WI) feedback from $S$, then picking $\bsigma(2)$ according to WI feedback from $S \setminus \{\bsigma(1)\}$, and so on, until all elements from $S$ are exhausted, or, in other words, successively sampling $|S|$ winners from $S$ according to the PL model, without replacement.

A feedback model that generalizes the types of feedback above is:

\item \textbf{Top-$m$ ranking of items (TR-$m$ or TR):} The environment returns a ranking of only $m$ items from among $S$, i.e., the environment first draws a full ranking $\bsigma$ over $S$ according to Plackett-Luce as in {\bf FR} above, and returns the first $m$ rank elements of $\bsigma$, i.e., $(\bsigma(1), \ldots, \bsigma(m))$. It can be seen that for each permutation $\sigma$ on a subset $S_m \subset S$, $|S_m| = m$, we must have $Pr(\bsigma = \sigma|S) = \prod_{i = 1}^{m}\frac{{\theta_{\sigma(i)}}}{\sum_{j = i}^{m}\theta_{\sigma(j)} + \sum_{j \in S \setminus S_m}\theta_{\sigma(j)}}$. Generating such a $\bsigma$ is also equivalent to successively sampling $m$ winners from $S$ according to the PL model, without replacement. It follows that {\bf TR} reduces to {\bf FR} when $m=k=|S|$ and to {\bf WI} when $m = 1$.

\end{itemize}

\subsection{Performance Objective: Correctness and Sample Complexity} 
\label{sec:obj}

Suppose $\btheta \equiv (\theta_1, \ldots, \theta_n)$ and $k \leq n$ define a BB-PL instance with best arm $i^* = 1$, and $0 < \epsilon \leq \frac{1}{2}, 0 < \delta \leq 1$ are given constants. An arm $i \in [n]$ is said to be $\epsilon$-optimal\footnote{informally, a `near-best' arm} if the probability that $i$ beats $1$ is over $\frac{1}{2} - \epsilon$, i.e., if $p_{i1}:= Pr(i|\{1,i\}) > \frac{1}{2} - \epsilon$. A sequential algorithm that operates in this BB-PL instance, using feedback from an appropriate subset-wise feedback model (e.g., WI, FR or TR), is said to be $(\epsilon,\delta)$-{PAC} if (a) it stops and outputs an arm $I \in [n]$ after a finite number of decision rounds (subset plays) with probability $1$, and (b) the probability that its output $I$ is an $\epsilon$-optimal arm is at least $1-\delta$, i.e, $Pr(\text{$I$ is $\epsilon$-optimal}) \geq 1-\delta$. Furthermore, by {\em sample complexity} of the algorithm, we mean the expected time (number of decision rounds) taken by the algorithm to stop. 


Note that $p_{ij} > \frac{1}{2} + \epsilon \Leftrightarrow \frac{\theta_i}{\theta_j} > \frac{1/2 + \epsilon}{1/2 - \epsilon}, \, \forall i,j \in [n]$, so the score parameter $\theta_i$ of a near-best item must be at least $\frac{1/2-\epsilon}{1/2+\epsilon}$ times $\theta_1$.


%% file: res_wi_lb.tex
\section{Analysis with Winner Information (WI) feedback}
\label{sec:res_wi}  
In this section we consider the {PAC-WI} goal with the {WI} feedback information model in BB-PL instances of size $n$ with playable subset size $k$. We start by showing that a sample complexity-lower bound for any $(\epsilon,\delta)$-{PAC} algorithm with {WI} feedback is $\Omega\bigg( \frac{n}{\epsilon^2} \ln \frac{1}{\delta}\bigg)$ (Theorem \ref{thm:lb_plpac_win}). This bound is independent of $k$, implying that playing a dueling game ($k=2$) is as good as the battling game as the extra flexibility of $k$-subsetwise feedback does not result in a faster learning rate. We next propose {two algorithms} for $(\epsilon,\delta)$-{PAC}, with WI feedback, with optimal (upto a logarithmic factor) sample complexity of $O(\frac{n}{\epsilon^2}\ln \frac{k}{\delta})$ (Section \ref{sec:algo_wi}). We also analyze a slightly different setting allowing the learner to play subsets $S_t$ of any size $1,2, \ldots, k$, rather than a fixed size $k$ -- this gives somewhat more flexibility to the learner, resulting in algorithms with improved sample complexity guarantees of $O(\frac{n}{\epsilon^2}\ln \frac{1}{\delta})$, without the $\ln k$ dependency as before (Section \ref{sec:reswi_anyk}).

\subsection{Lower Bound for Winner Information (WI) feedback}
\label{sec:lb_wi}

\begin{restatable}[Lower bound on Sample Complexity with WI feedback]{thm}{lbwin}
\label{thm:lb_plpac_win}
Given $\epsilon \in (0,\frac{1}{\sqrt{8}}]$ and $\delta \in (0,1]$, and an $(\epsilon,\delta)$-PAC algorithm $A$ for BB-PL with feedback model {WI}, there exists a PL instance $\nu$ such that the sample complexity of $A$ on $\nu$ is at least
$\Omega\bigg( \frac{n}{\epsilon^2} \ln \frac{1}{2.4\delta}\bigg).
$
\end{restatable}
\vspace*{-15pt}
\begin{proof}\textbf{(sketch)}.
The argument is based on a change-of-measure argument (Lemma  $1$) of \cite{Kaufmann+16_OnComplexity}, restated below for convenience: 

Consider a multi-armed bandit (MAB) problem with $n$ arms or actions $\cA = [n]$. At round $t$, let $A_t$ and $Z_t$ denote the arm played and the observation (reward) received, respectively. Let $\cF_t = \sigma(A_1,Z_1,\ldots,A_t,Z_t)$ be the sigma algebra generated by the trajectory of a sequential bandit algorithm upto round $t$.
\begin{restatable}[Lemma $1$, \cite{Kaufmann+16_OnComplexity}]{lem}{gar16}
\label{lem:gar16}
Let $\nu$ and $\nu'$ be two bandit models (assignments of reward distributions to arms), such that $\nu_i ~(\text{resp.} \,\nu'_i)$ is the reward distribution of any arm $i \in \cA$ under bandit model $\nu ~(\text{resp.} \,\nu')$, and such that for all such arms $i$, $\nu_i$ and $\nu'_i$ are mutually absolutely continuous. Then for any almost-surely finite stopping time $\tau$ with respect to $(\cF_t)_t$,
\vspace*{-5pt}
\begin{align*}
\sum_{i = 1}^{n}\E_{\nu}[N_i(\tau)]KL(\nu_i,\nu_i') \ge \sup_{\cE \in \cF_\tau} kl(Pr_{\nu}(\cE),Pr_{\nu'}(\cE)),
\end{align*}
where $kl(x, y) := x \log(\frac{x}{y}) + (1-x) \log(\frac{1-x}{1-y})$ is the binary relative entropy, $N_i(\tau)$ denotes the number of times arm $i$ is played in $\tau$ rounds, and $Pr_{\nu}(\cE)$ and $Pr_{\nu'}(\cE)$ denote the probability of any event $\cE \in \cF_{\tau}$ under bandit models $\nu$ and $\nu'$, respectively.
\end{restatable}

To employ this result, note that in our case, each bandit instance corresponds to an instance of the BB-PL problem with the arm set containing all subsets of $[n]$ of size $k$: $\cA = \{S = (S(1), \ldots S(k)) \subseteq [n] ~|~ S(i) < S(j), \, \forall i < j\}$. The key part of our proof relies on carefully crafting a true instance, with optimal arm $1$, and a family of slightly perturbed alternative instances $\{\bnu^a: a \neq 1\}$, each with optimal arm $a \neq 1$.

We choose the true problem instance $\bnu^1$ as the Plackett-Luce model with parameters
\begin{align*}
 \theta_j = \theta\bigg( \frac{1}{2} - \epsilon\bigg), \forall j \in [n]\setminus \{1\}, \text{ and } \theta_1 = \theta\bigg( \frac{1}{2} + \epsilon \bigg), \quad (\text{true instance})
\end{align*}
for some $\theta \in \R_+, ~\epsilon > 0$. Corresponding to each suboptimal item $a \in [n]\setminus \{1\}$, we now define an alternative problem instance $\bnu^a$ as the Plackett-Luce model with parameters
\begin{align*}
 \theta'_j = \theta\bigg( \frac{1}{2} - \epsilon \bigg)^2, \forall j \in [n]\setminus \{a,1\}, \, \theta'_1 = \theta\bigg( \frac{1}{4} - \epsilon^2 \bigg), \theta'_a = \theta\bigg( \frac{1}{2} + \epsilon \bigg)^2 \quad (\text{alternative instance}).
\end{align*}

The result of Theorem \ref{thm:lb_plpac_win} is now obtained by applying Lemma \ref{lem:gar16} on pairs of problem instances $(\nu, \nu'^{(a)})$, with suitable upper bounds on the KL-divergence terms, and the observation that $kl(\delta,1-\delta) \geq \ln \frac{1}{2.4\delta}$. The complete proof is given in Appendix \ref{app:wilb}.
\end{proof}

\vspace*{-10pt}

\begin{rem}
Theorem \ref{thm:lb_plpac_win} shows, rather surprisingly, that the PAC sample complexity of identifying a near-optimal item with only winner feedback information from $k$-size subsets, does not reduce with $k$, implying that there is no reduction in hardness of learning from the pairwise comparisons case ($k = 2$). On one hand, one may  expect to see improved sample complexity as the number of items being simultaneously tested in each round is large ($k$). On the other hand, the sample complexity could also worsen, since it is intuitively `harder' for a good (near-optimal) item to win and show itself, in just a single winner draw, against a large population of $k-1$ other competitors. The result, in a sense, formally establishes that the former advantage is nullified by the latter drawback. A somewhat more formal, but heuristic, explanation for this phenomenon is that the number of bits of information that a single winner draw from a size-$k$ subset provides is $O(\ln k)$, which is not significantly larger than when $k > 2$, thus an algorithm cannot accumulate significantly more information per round compared to the pairwise case. 
\end{rem}

%% file: res_wi_algo.tex
\subsection{Algorithms for Winner Information (WI) feedback model}
\label{sec:algo_wi}
This section describes our proposed algorithms for the $(\epsilon,\delta)$-{PAC} objective with winning item ({WI}) feedback. 


\vspace*{5pt}
\noindent {\bf Principles of algorithm design.} The key idea on which all our learning algorithms are based is that of maintaining estimates of the pairwise win-loss probabilities $p_{ij} = Pr(i|{i,j})$ in the Plackett-Luce model. This helps circumvent an $O(n^k)$ combinatorial explosion that would otherwise result if we directly attempted to estimate probability distributions for each possible $k$-size subset. However, it is not obvious if consistent and tight pairwise estimates can be constructed in a general subset-wise choice model, but the special form of the Plackett-Luce model again comes to our rescue. The IIA property that the PL model enjoys, allows for accurate pairwise estimates via interpretation of partial preference feedback as a set of pairwise preferences, e.g., a winner $a$ sampled from among ${a,b,c}$ is interpreted as the pairwise preferences $a \succ b$, $a \succ c$. Lemma \ref{lem:pl_simulator} formalizes this property and allows us to use pairwise win/loss probability estimators with explicit confidence intervals for them.

  

\vspace*{5pt}
\noindent
\textbf{Algorithm \ref{alg:trace_bb}: (\algtrc).}
Our first algorithm \algtrc\, is based on the simple idea of tracing the empirical best item--specifically, it maintains a running winner $r_\ell$ at every iteration $\ell$, making it battle with a set of $k-1$ arbitrarily chosen items. After battling long enough (precisely, for $\frac{2k}{\epsilon^2}\ln \frac{2n}{\delta}$ many rounds), if the empirical winner $c_\ell$ turns out to be more than $\frac{\epsilon}{2}$-favorable than the running winner $r_\ell$, in term of its pairwise preference score: $\hp_{c_\ell,r_\ell} > \frac{1}{2} + \frac{\epsilon}{2}$, then $c_\ell$ replaces $r_\ell$, or else $r_\ell$ retains its place and status quo ensues. 

\vspace*{-3pt}
\begin{restatable}[\algtrc:  Correctness and Sample Complexity with WI]{thm}{ubtrc}
\label{thm:trace_best}
\algtrc\, (Algorithm \ref{alg:trace_bb}) is $(\epsilon,\delta)$-{PAC} with sample complexity $O(\frac{n}{\epsilon^2} \log\frac{n}{\delta})$.
\end{restatable}

\vspace*{-2pt}
\begin{proof}\textbf{(sketch)}.
The main idea is to retain an estimated best item as a `running winner' $r_\ell$, and compare it with the `empirical best item' $c_\ell$ of $\cA$ at every iteration $\ell$. The crucial observation lies in noting that at any iteration $\ell$, $r_\ell$ gets updated as follows:

\vspace*{-10pt}
\begin{restatable}[]{lem}{lemtrc}
\label{lem:c_vs_r}
At any iteration $\ell = 1,2 \ldots \big \lfloor \frac{n}{k-1} \big \rfloor$, with probability at least $(1-\frac{\delta}{2n})$, Algorithm \ref{alg:trace_bb} retains $r_{\ell+1} \leftarrow r_\ell$ if $p_{c_\ell r_\ell } \le \frac{1}{2}$, and sets $r_{\ell+1} \leftarrow c_\ell$ if $p_{c_\ell r_\ell} \ge \frac{1}{2} + \epsilon$.
\end{restatable}

This leads to the claim that between any two successive iterations $\ell$ and $\ell + 1$, we must have, with high probability, that
$p_{r_{\ell+1}r_\ell} \ge \frac{1}{2} \text{ and, } p_{r_{\ell+1}c_\ell} \ge \frac{1}{2} - \epsilon,
$ showing that the estimated `best' item $r_\ell$ can only get improved per iteration as $p_{r_{\ell+1}r_\ell} \ge \frac{1}{2}$ (with high probability at least $1-\frac{(k-1)\delta}{2n}$).
Repeating this above argument for each iteration $\ell \in \big \lfloor \frac{n}{k-1} \big \rfloor$ results in the desired correctness guarantee of $p_{r_*1} \ge \frac{1}{2} - \epsilon$.
The sample complexity bound follows easily by noting the total number of possible iterations can be at most $\lceil  \frac{n}{k-1} \rceil$, with the per-iteration sample complexity being $t = \frac{2k}{\epsilon^2}\ln \frac{2n}{\delta}$.
\end{proof}

\vspace*{-25pt}
\begin{center}
\begin{algorithm}[H]
   \caption{\textbf{\algtrc} }
   \label{alg:trace_bb}
\begin{algorithmic}[1]
   \STATE {\bfseries Input:} 
   \STATE ~~~ Set of items: $[n]$, Subset size: $n \geq k > 1$
   \STATE ~~~ Error bias: $\epsilon >0$, Confidence parameter: $\delta >0$
   \STATE {\bfseries Initialize:} 
   \STATE ~~~ $r_1 \leftarrow $  Any (random) item from $[n]$, $\cA \leftarrow $ Randomly select $(k-1)$ items from $[n]\setminus \{r_1\}$
   \STATE ~~~ Set $\cA \leftarrow \cA \cup \{r_1\}$, and $S \leftarrow [n]\setminus \cA$  
   \WHILE {$\ell = 1, 2, \ldots$}
   	\STATE Play the set $\cA$ for $t:= \frac{2k}{\epsilon^2}\ln \frac{2n}{\delta}$ rounds
   \STATE $w_i \leftarrow$ Number of times $i$ won in $t$ plays of $\cA$, $\forall i \in \cA$
   \STATE Set $c_\ell \leftarrow \underset{i \in \cA}{\text{argmax}}~w_i$, and $\hp_{ij} \leftarrow \frac{w_i}{w_i + w_j}, \, \forall i,j \in \cA, i \neq j$
   \STATE \textbf{if} $\hp_{c_\ell,r_\ell} > \frac{1}{2} + \frac{\epsilon}{2} $, \textbf{then} set $r_{\ell+1} \leftarrow c_\ell$; \textbf{else} $r_{\ell+1} \leftarrow r_\ell$ 	
	\IF {$(S == \emptyset)$}
	\STATE Break (exit the while loop)    
	\ELSIF {$|S| < k-1$}
	\STATE $\cA \leftarrow$ Select $(k-1-|S|)$ items from $\cA \setminus \{r_\ell\}$ uniformly at random
	\STATE $\cA \leftarrow \cA \cup \{r_\ell\} \cup S$; and $S \leftarrow \emptyset$
	\ELSE
	\STATE $\cA	\leftarrow $ Select $(k-1)$ items from $S$ uniformly at random
	\STATE $\cA \leftarrow \cA \cup \{r_\ell\}$ and $S\leftarrow S \setminus \cA$
	\ENDIF
	\ENDWHILE
   \STATE {\bfseries Output:} $r_* = r_\ell$ as the $\epsilon$-optimal item 
\end{algorithmic}
\end{algorithm}
\vspace{-2pt}
\end{center}
\vspace{-15pt}

\begin{rem}
The sample complexity of \algtrc, is order wise optimal when $\delta < \frac{1}{n}$, as follows from our derived lower bound guarantee (Theorem \ref{thm:lb_plpac_win}). 
\end{rem}
\vspace*{-10pt}

When $\delta > \frac{1}{n}$, the sample complexity guarantee of \algtrc\, is off by a factor of $\ln n$. We now propose another algorithm, \algdiv\, (Algorithm \ref{alg:div_bb}) that enjoys an $(\epsilon,\delta)$-{PAC} sample complexity of $O\left(\frac{n}{\epsilon^2}\ln \frac{k}{\delta}\right)$. 

\vspace*{5pt}

\noindent
\textbf{Algorithm \ref{alg:div_bb}: (\algdiv).}
\algdiv\, first divides the set of $n$ items into groups of size $k$, and plays each group long enough so that a good item in the group stands out as the empirical winner with high probability (Line $11$). It then retains the empirical winner per group (Line $13$) and recurses on the retained set of the winners, until it is left with only a single item, which is finally declared as the $\epsilon$-optimal item. \emph{The pseudo code of \algdiv\, is given in Appendix \ref{app:wiub_div}.}



\vspace*{-7pt}
\begin{restatable}[\algdiv: \hspace*{-5pt} Correctness and Sample Complexity with WI]{thm}{ubdiv}
\label{thm:div_bb}
\hspace*{-7pt}
\algdiv\, (Algorithm \ref{alg:div_bb}) is $(\epsilon,\delta)$-{PAC} with sample complexity $O(\frac{n}{\epsilon^2} \log \frac{k}{\delta})$.
\end{restatable}

\vspace*{-7pt}
\begin{proof}\textbf{(sketch)}.
The crucial observation here is that at any iteration $\ell$, for any set $\cG_g$ ($g = 1,2,\ldots G$), the item $c_g$ retained by the algorithm is likely to be not more than $\epsilon_\ell$-worse than the best item of the set $\cG_g$, with probability at least $(1-\delta_\ell)$. Precisely, we show that:

\vspace*{-7pt}
\begin{restatable}[]{lem}{lemdiv}
\label{lem:div_bb} 
At any iteration $\ell$, for any $\cG_g$, if $i_g := \underset{i \in \cG_g}{{\arg\max}}~\theta_i$, then with probability at least $(1-\delta_\ell)$, $p_{c_gi_g} > \frac{1}{2} - \epsilon_\ell$.
\end{restatable} 

This guarantees that, between any two successive rounds $\ell$ and $\ell+1$, we do not lose out by more than an additive factor of $\epsilon_\ell$ in terms of highest score parameter of the remaining set of items. Aggregating this claim over all iterations can be made to show  that $p_{r_*1} > \frac{1}{2} - \epsilon$, as desired. 
The sample complexity bound follows by carefully summing the total number of times ($t= \frac{k}{2\epsilon_\ell^2}\ln \frac{k}{\delta_\ell}$) a set $\cG_g$ is played per iteration $\ell$, with the maximum number of possible iterations being $\lceil  \ln_k n \rceil$.
\end{proof}
\vspace*{-15pt}

\begin{rem}
\label{rem:lb_loose}
The sample complexity of \algdiv\, is order-wise optimal in the `small-$\delta$' regime $\delta \ll \frac{1}{k}$ by the lower bound result (Theorem \ref{thm:lb_plpac_win}). However, for the `moderate-$\delta$' regime $\delta \gtrapprox \frac{1}{k}$, 
we conjecture that the lower bound is loose by an additive factor of $\frac{n \ln k}{\epsilon^2}$, i.e., that a improved lower bound of $\Omega(\frac{n}{\epsilon^2} \log \frac{k}{\delta})$ holds. This is primarily because we believe that the error probability $\delta$ of any typical, label-invariant PAC algorithm ought to be distributed roughly uniformly across misidentification of all the items, allowing us to use $\delta/k$ instead of $\delta$ on the right hand side of the change-of-measure inequalities of Lemma \ref{lem:gar16}, resulting in the improved quantity $\ln (k/2.4 \delta)$. This is perhaps in line with recent work in multi-armed bandits \citep{SimJamRec17:moderate} that points to an increased difficulty of PAC identification in the moderate-confidence regime.
\end{rem}

We now consider a variant of the BB-PL decision model which allows the learner to play sets of any size $1,2,\ldots, k$, instead of a fixed size $k$. In this setting, we are indeed able to design an $(\epsilon,\delta)$-{PAC} algorithm that enjoys an order-optimal $O(\frac{n}{\epsilon^2}\ln \frac{1}{\delta})$ sample-complexity.

\subsection{BB-PL2: A slightly different battling bandit decision model}
\label{sec:reswi_anyk}

The new winner information feedback model {BB-PL-2} is formally defined as follows: At each round $t$, here the learner is allowed to select a set $S_t \subseteq [n]$ of size $2,3,\ldots,$ upto $k$. Upon receiving any set $S_t$, the environment returns the index of the winning item as  $I \in [|S|]$ such that,
$
\label{eq:prob_win2}
\P(I = i|S) = \frac{{\theta_{S(i)}}}{\sum_{j = 1}^{|S|}\theta_{S(j)}} ~~\forall i \in [|S|].
$

\vspace*{8pt}
\noindent \textbf{On applying existing PAC-Dueling-Bandit strategies.} Note that given the flexibility of playing sets of any size, one might as well hope to apply the PAC-Dueling Bandit algorithm \emph{PLPAC}$(\epsilon,\delta)$ of \cite{Busa_pl} which plays only pairs of items per round. However, their algorithm is shown to have a sample complexity guarantee of $O\Big( \frac{n}{\epsilon^2} \ln \frac{n}{\epsilon\delta} \Big)$, which is suboptimal by an additive $O\Big(\frac{n}{\epsilon^2} \ln \frac{n}{\epsilon}\Big)$ as our results will show. A similar observation holds for the \emph{Beat-the-Mean} (BTM) algorithm of \citet{BTM}, which in fact has a even worse sample complexity guarantee of $O\Big( \frac{n}{\epsilon^2} \ln \big( \frac{n}{\epsilon^2\delta} \ln \frac{n}{\delta} \big) \Big)$.

\setcounter{algorithm}{2}
\vspace*{-10pt}
\begin{center}
\begin{algorithm}[H]
   \caption{\textbf{\algmed}}
   \label{alg:half_bb}
\begin{algorithmic}[1]
   \STATE {\bfseries Input:} 
   \STATE ~~~ Set of items: $[n]$, Maximum subset size: $n \geq k > 1$
   \STATE ~~~ Error bias: $\epsilon >0$, Confidence parameter: $\delta >0$
   \STATE {\bfseries Initialize:} 
   \STATE ~~~ $S \leftarrow [n]$, $\epsilon_0 \leftarrow \frac{\epsilon}{4}$, and $\delta_0 \leftarrow {\delta}$  
   \STATE ~~~ Divide $S$ into $G: = \lceil \frac{n}{k} \rceil$ sets $\cG_1, \cG_2, \cdots \cG_G$ such that $\cup_{j = 1}^{G}\cG_j = S$ and $\cG_{j} \cap \cG_{j'} = \emptyset, ~\forall j,j' \in [G]$, where $|G_j| = k,\, \forall j \in [G-1]$
   \WHILE{$\ell = 1,2, \ldots$}
   \STATE $S \leftarrow \emptyset$, $\delta_\ell \leftarrow \frac{\delta_{\ell-1}}{2}, \epsilon_\ell \leftarrow \frac{3}{4}\epsilon_{\ell-1}$
   \FOR {$g = 1,2, \cdots G$}
   \STATE Play $\cG_g$ for $t:= \frac{k}{2\epsilon_\ell^2}\ln \frac{4}{\delta_\ell}$ rounds
   \STATE $w_i \leftarrow$ Number of times $i$ won in $t$ plays of $\cG_g$, $\forall i \in \cG_g$
	\STATE Set $h_{g} \leftarrow ~\text{Median}(\{w_i \mid i \in \cG_g\})$, and $S \leftarrow S \cup \{i \in \cG_g \mid w_i \ge w_{h_g}\}$
   \ENDFOR
	\IF {$|S| == 1$}
  	    \STATE Break (exit the while loop)    
    \ELSE
        \STATE Divide $S$ into $G: = \big \lceil \frac{|S|}{k} \big \rceil$ sets $\cG_1, \cG_2, \cdots \cG_G$ such that $\cup_{j = 1}^{G}\cG_j = S$ and $\cG_{j} \cap \cG_{j'} = \emptyset, ~\forall j,j' \in [G]$, where $|G_j| = k,\, \forall j \in [G-1]$
    \ENDIF
   \ENDWHILE
   \STATE {\bfseries Output:} $r_*$ as the $\epsilon$-optimal item, where $S = \{r_*\}$
\end{algorithmic}
\end{algorithm}
\vspace{-5pt}
\end{center}
\vspace*{-10pt}


\noindent
\textbf{Algorithm \ref{alg:half_bb}: \algmed.}
We here propose a Median-Elimination-based approach \citep{Even+06} which is shown to run with \emph{optimal sample complexity} $O(\frac{n}{\epsilon^2}\ln \frac{1}{\delta})$ rounds (Theorem \ref{thm:half_bb}). 
(Note that an $\Omega(\frac{n}{\epsilon^2}\ln \frac{1}{\delta})$ fundamental limit on PAC sample complexity for BB-PL2-WI can easily be derived using an argument along the lines of Theorem \ref{thm:lb_plpac_win}; we omit the explicit derivation.)
The name \algmed\, for the algorithm is because it is based on the idea of dividing the set of items into two partitions with respect to the \emph{empirical median item} and retaining the `better half'. Specifically, it first divides the entire item set into groups of size $k$, and plays each group for a fixed number of times. After this step, only the items that won more than the empirical median $h_g$ are retained and rest are discarded. The algorithm recurses until it is left with a single item. The intuition here is that some \textit{$\epsilon$-best item} is always likely to beat the group median and can never get wiped off. 


\begin{restatable}[\algmed:  Correctness and Sample Complexity with WI]{thm}{ubmed}
\label{thm:half_bb}
\algmed\, (Algorithm \ref{alg:half_bb}) is $(\epsilon,\delta)$-{PAC} with {sample complexity} $O\big(\frac{n}{\epsilon^2} \log \frac{1}{\delta}\big)$.
\end{restatable}

\vspace*{-2pt}

\begin{proof}\textbf{(sketch)}.
The sample complexity bound follows by carefully summing the total number of times ($t = \frac{k}{2\epsilon_\ell^2}\ln \frac{1}{\delta_\ell}$) a set $\cG_g$ is played per iteration $\ell$, with the maximum number of possible iterations being $\lceil  \ln n \rceil$ (this is because the size of the set $S$ of remaining items gets halved at each iteration as it is pruned with respect to its median).
The key intuition in proving the correctness property of \algmed\, lies in showing that at any iteration $\ell$, \algmed\, always carries forward at least one `near-best' item to the next iteration $\ell+1$. 


\begin{restatable}[]{lem}{lemmed}
\label{lem:hv_pij} 
At any iteration $\ell$, for any set $\cG_g$, let $i_g \leftarrow \underset{i \in \cG_g}{{\arg\max}}~\theta_i$, and consider any suboptimal item $b \in \cG_g$ such that $p_{b i_g} < \frac{1}{2} - \epsilon_\ell$. Then with probability at least $\big(1-\frac{\delta_\ell}{4}\big)$, the empirical win count of $i_g$ lies above that of $b$, i.e. $w_{i_g} \ge w_b$ (equivalently $\hp_{i_gb} = \frac{w_{i_g}}{w_{i_g} + w_b} \ge \frac{1}{2}$).
\end{restatable} 
\vspace*{-0pt}

Using the property of the median element $h_g$ along with Lemma \ref{lem:hv_pij} and Markov's inequality, we show that we do not lose out more than an additive factor of $\epsilon_\ell$ in terms of highest score $\theta_i$ of the remaining set of items between any two successive iterations $\ell$ and $\ell+1$. This finally leads to the desired $(\epsilon,\delta)$-PAC correctness of \algmed.
\end{proof}

\begin{rem}
Theorem \ref{thm:half_bb} shows that the sample complexity guarantee of \algmed\ improves over the that of existing PLPAC algorithm for the same objective in dueling bandit setup ($k=2$), which was shown to be $O\big(\frac{n}{\epsilon^2} \log \frac{n}{\epsilon\delta}\big)$ (see Theorem $3$, \cite{Busa_pl}), and also the $O\Big( \frac{n}{\epsilon^2} \ln \big( \frac{n}{\epsilon^2\delta} \ln \frac{n}{\delta} \big) \Big)$ complexity of BTM algorithm \citep{BTM} for dueling feedback from any pairwise preference matrix with relaxed stochastic transitivity and stochastic triangle inequality (of which PL model is a special case).
\end{rem}

%% file: res_fr_lb.tex
\vspace*{-15pt}

\section{Analysis with Top Ranking (TR) feedback}
\label{sec:res_fr}

We now proceed to analyze the {BB-PL} problem with \textit{Top-$m$ Ranking} ({TR}) feedback (Section \ref{sec:feed_mod}). We first show that unlike WI feedback, the sample complexity lower bound here scales as $\Omega\bigg( \frac{n}{m\epsilon^2} \ln \frac{1}{\delta}\bigg)$ (Theorem \ref{thm:lb_pacpl_rnk}), which is a factor ${m}$ smaller than that in Thm. \ref{thm:lb_plpac_win} for the WI feedback model. At a high level, this is because TR reveals the preference information of $m$ items per feedback step (round of battle), as opposed to just a single (noisy) information sample of the winning item {(WI)}. Following this, we also present two algorithms for this setting which are shown to enjoy an optimal (upto logarithmic factors) sample complexity guarantee of $O\bigg( \frac{n}{m\epsilon^2} \ln \frac{k}{\delta}\bigg)$ (Section \ref{sec:alg_fr}).

\vspace*{-5pt}

\subsection{Lower Bound for Top-$m$ Ranking (TR) feedback}
\label{sec:lb_fr}


\begin{restatable}[Sample Complexity Lower Bound for TR]{thm}{lbrnk}
\label{thm:lb_pacpl_rnk}
Given $\epsilon \in (0,\frac{1}{\sqrt{8}}]$ and $\delta \in (0,1]$, and an $(\epsilon,\delta)$-PAC algorithm $A$ with top-$m$ ranking ({TR}) feedback ($2 \le m \le k$), there exists a PL instance $\nu$ such that the expected sample complexity of $A$ on $\nu$ is at least
$\Omega\bigg( \frac{n}{m\epsilon^2} \ln \frac{1}{2.4\delta}\bigg)
$.
\end{restatable}

\vspace*{-15pt}

\begin{rem}
The sample complexity lower for {PAC-WI} objective for BB-PL with top-$m$ ranking ({TR}) feedback model is $\frac{1}{m}$-times that of the {WI} model (Thm. \ref{thm:lb_plpac_win}). Intuitively, revealing a ranking on $m$ items in a $k$-set provides about $\ln \left({k \choose m} m!\right) = O(m \ln k)$ bits of information per round, which is about $m$ times as large as that of revealing a single winner, yielding an acceleration of $m$.  
\end{rem}


\begin{cor}
\label{cor:lb_pacpl_rnk}
Given $\epsilon \in (0,\frac{1}{\sqrt{8}}]$ and $\delta \in (0,1]$, and an $(\epsilon,\delta)$-PAC algorithm $A$ with full ranking ({FR}) feedback ($m = k$), there exists a PL instance $\nu$ such that the expected sample complexity of $A$ on $\nu$ is at least
$\Omega\bigg( \frac{n}{k\epsilon^2} \ln \frac{1}{2.4\delta}\bigg)
$.
\end{cor}

%% file: res_fr_algo.tex
\subsection{Algorithms for Top-$m$ Ranking (TR) feedback model}
\label{sec:alg_fr}

This section presents two algorithms for $(\epsilon,\delta)$-{PAC} objective for \textit{BB-PL} with top-$m$ ranking feedback. We achieve this by generalizing our earlier two proposed algorithms (see Algorithm \ref{alg:trace_bb} and \ref{alg:div_bb}, Sec. \ref{sec:algo_wi} for {WI} feedback) to the top-$m$ ranking ({TR}) feedback mechanism.
\footnote{Our third algorithm \algmed\, is not applicable to {TR} feedback as it allows the learner to play sets of sizes $1,2,3,\ldots \text{ upto } k$, whereas the {TR} feedback is defined only when the size of the subset played is at least $m$. The lower bound analysis of Theorem \ref{thm:lb_pacpl_rnk} also does not apply if sets of size less than $m$ is allowed.}

\vspace*{5pt}
\noindent
\textbf{Rank-Breaking.}
{The main trick we use in modifying the above algorithms for TR feedback is \textit{Rank Breaking} \citep{AzariRB+14}, which essentially extracts pairwise comparisons from multiwise (subsetwise) preference information}. Formally, given any set $S$ of size $k$, if $\bsigma \in \bSigma_{S_m},\, (S_m \subseteq S,\, |S_m|=m)$ denotes a possible top-$m$ ranking of $S$, the \textit{Rank Breaking} subroutine considers each item in $S$ to be beaten by its preceding items in $\bsigma$ in a pairwise sense. For instance, given a full ranking of a set of $4$ elements $S = \{a,b,c,d\}$, say $b \succ a \succ c \succ d$, Rank-Breaking generates the set of $6$ pairwise comparisons: $\{(b\succ a), (b\succ c), (b\succ d), (a\succ c), (a\succ d), (c\succ d)\}$. Similarly, given the ranking of only $2$ most preferred items say $b \succ a$, it yields the $5$ pairwise comparisons $(b, a\succ c),(b,a\succ d)$ and $(b\succ a)$ etc. See Algorithm \ref{alg:updt_win} for detailed description of the Rank-Breaking procedure.

\vspace*{-3pt}

\begin{restatable}[Rank-Breaking Update]{lem}{lemrb}
\label{lem:rb}
Consider any subset $S \subseteq [n]$ with $|S| = k$. Let $S$ be played for $t$ rounds of battle, and let $\bsigma_\tau \in \bSigma_{S^\tau_m}, \, (S^\tau_m \subseteq S,\, |S^\tau_m| = m)$, denote the TR feedback at each round $\tau \in [t]$. For each item $i \in S$, let $q_{i}: = \sum_{\tau = 1}^{t}\1(i \in S^\tau_m)$ be the number of times $i$ appears in the top-$m$ ranked output in $t$ rounds. Then, the most frequent item(s) in the top-$m$ positions must appear at least $\frac{mt}{k}$ times, i.e. $\max_{i \in S}q_i \ge \frac{mt}{k}$.
\end{restatable}

\vspace*{-19pt}
\begin{center}
\begin{algorithm}[H]
   \caption{\algupdt\, (for updating the pairwise win counts $w_{ij}$ for {TR} feedback)}
   \label{alg:updt_win}
\begin{algorithmic}[1]
   \STATE {\bfseries Input:} 
   STATE ~~~ Subset $S \subseteq [n]$, $|S| = k$ ($n\ge k$) 
   \STATE ~~~ A top-$m$ ranking $\bsigma \in \bSigma_{S_m}$, $S_m \subseteq [n], \, |S_m| = m$
   \STATE ~~~ Pairwise (empirical) win-count $w_{ij}$ for each item pair $i,j \in S$
   \WHILE {$\ell = 1, 2, \ldots m$}
   	\STATE Update $w_{\sigma(\ell)i} \leftarrow w_{\sigma(\ell)i} + 1$, for all $i \in S \setminus\{\sigma(1),\ldots,\sigma(\ell)\}$
	\ENDWHILE
\end{algorithmic}
\end{algorithm}
\vspace{-5pt}
\end{center}
\vspace*{-5pt}

\noindent
\textbf{Proposed Algorithms for TR feedback.}
The formal descriptions of our two algorithms, \algtrc\, and \algdiv\,, generalized to the setting of {TR} feedback, are given as Algorithm \ref{alg:trace_bb_mod}  and Algorithm \ref{alg:div_bb_mod} respectively.
They essentially maintain the empirical pairwise preferences $\hp_{ij}$ for each pair of items $i,j$ by applying \textit{Rank Breaking} on the {TR} feedback $\bsigma$ after each round of battle. 
Of course in general, \emph{Rank Breaking} may lead to arbitrarily inconsistent estimates of the underlying model parameters \citep{Az+12}. However, owing to the {\it IIA property} of the Plackett-Luce model, we get clean concentration guarantees on $p_{ij}$ using Lemma \ref{lem:pl_simulator}. This is precisely the idea used for obtaining the $\frac{1}{m}$ factor improvement in the sample complexity guarantees of our  proposed algorithms along with Lemma \ref{lem:rb} (see proofs of Theorem \ref{thm:trace_best_fr} and \ref{thm:div_bb_fr}).

\begin{restatable}[\algtrc:  Correctness and Sample Complexity with TR]{thm}{ubtrcfr}
\label{thm:trace_best_fr}
\hspace*{-5pt} With top-$m$ ranking (TR) feedback model, \algtrc\, (Algorithm \ref{alg:trace_bb_mod}) is $(\epsilon,\delta)$-{PAC} with sample complexity $O(\frac{n}{m\epsilon^2} \log\frac{n}{\delta})$.
\end{restatable}

\vspace*{-10pt}
\begin{restatable}[\algdiv: \hspace*{-5pt} Correctness and Sample Complexity with TR]{thm}{ubdivfr}
\label{thm:div_bb_fr}
\hspace*{-7pt} With top-$m$ ranking (TR) feedback model, \algdiv\, (Algorithm \ref{alg:div_bb_mod}) is $(\epsilon,\delta)$-{PAC} with sample complexity $O(\frac{n}{m\epsilon^2} \log \frac{k}{\delta})$.
\end{restatable}

\begin{rem}
The sample complexity bounds of the above two algorithms are $\frac{1}{m}$ fraction lesser than their corresponding counterparts for WI feedback, as follows comparing Theorem \ref{thm:trace_best} vs. \ref{thm:trace_best_fr}, or Theorem \ref{thm:div_bb} vs. \ref{thm:div_bb_fr}, which admit a faster learning rate with TR feedback. Similar to the case with WI feedback, sample complexity of \algdiv\, is still orderwise optimal for any $\delta \le \frac{1}{k}$, as follows from the lower bound guarantee (Theorem \ref{thm:lb_pacpl_rnk}). However, we believe that the above lower bound can be tightened by a factor of $\ln k$ for 'moderate' $\delta \gtrapprox \frac{1}{k}$, for reasons similar to those stated in Remark \ref{rem:lb_loose}.
\end{rem}

%% file: conclusion.tex

\vspace*{-20pt}
\begin{center}
\begin{algorithm}[H]
   \caption{\textbf{\algtrc} (for {TR} feedback) }
   \label{alg:trace_bb_mod}
\begin{algorithmic}[1]
   \STATE {\bfseries Input:} 
   \STATE ~~~ Set of items: $[n]$, and subset size: $k > 2$ ($n \ge k \ge m$)
   \STATE ~~~ Error bias: $\epsilon >0$, and confidence parameter: $\delta >0$
   \STATE {\bfseries Initialize:} 
   \STATE ~~~ $r_1 \leftarrow $  Any (random) item from $[n]$, $\cA \leftarrow $ Randomly select $(k-1)$ items from $[n]\setminus \{r_1\}$
   \STATE ~~~ Set $\cA \leftarrow \cA \cup \{r_1\}$, and $S \leftarrow [n]\setminus \cA$  
   \WHILE {$\ell = 1, 2, \ldots$}
    \STATE Initialize pairwise (empirical) win-count $w_{ij} \leftarrow 0$, for each item pair $i,j \in \cA$
   	\FOR {$\tau = 1, 2, \ldots t\,(:= \frac{2k}{m\epsilon^2}\ln \frac{2n}{\delta})$}
   	\STATE Play the set $\cA$ (one round of battle)
   	\STATE Receive TR feedback: $\bsigma \in \bSigma_{\cA^\tau_m}$, where $\cA^\tau_m \subseteq \cA$ such that $|\cA^\tau_m| = m$ 
   	\STATE Update pairwise win-counts $w_{ij}$ of each item pair $i,j \in \cA$ using \algupdt$(\cA,\bsigma)$
   	\ENDFOR
   \STATE $B_\ell \leftarrow \text{argmax}\{ i \in \cA \mid  \sum_{j \in \cA\setminus\{i'\}}~\1\big(w_{ij} \ge w_{ji}\big) \}$, 
   \STATE $\hp_{ij} \leftarrow \frac{w_{ij}}{w_{ij} + w_{ji}}, \, \forall i,j \in \cA, i \neq j$
   \STATE \textbf{if} $\exists c_\ell \in B_\ell \text{ such that } \hp_{c_\ell,r_\ell} > \frac{1}{2} + \frac{\epsilon}{2} $, \textbf{then} set $r_{\ell+1} \leftarrow c_\ell$; \textbf{else} set $r_{\ell+1} \leftarrow r_\ell$ 	
	\IF {$(S == \emptyset)$}
	\STATE Break (go out of the while loop)
	\ELSIF {$|S| < k-1$}
	\STATE $\cA \leftarrow$ Randomly select $(k-1-|S|)$ items from $\cA \setminus \{r_\ell\}$
	\STATE $\cA \leftarrow \cA \cup \{r_\ell\} \cup S$; and $S \leftarrow \emptyset$
	\ELSE
	\STATE $\cA	\leftarrow $ Randomly select $(k-1)$ items from $S$
	\STATE $\cA \leftarrow \cA \cup \{r_\ell\}$ and $S\leftarrow S \setminus \cA$
	\ENDIF
	\ENDWHILE
   \STATE {\bfseries Output:} $r_* = r_\ell$ as the $\epsilon$-optimal item
\end{algorithmic}
\end{algorithm}
\end{center}

\begin{center}
\begin{algorithm}[H]
   \caption{\textbf{\algdiv} (for {TR} feedback) }
   \label{alg:div_bb_mod}
\begin{algorithmic}[1]
   \STATE {\bfseries Input:} 
   \STATE ~~~ Set of items: $[n]$, and subset size: $k > 2$ ($n \ge k \ge m$)
   \STATE ~~~ Error bias: $\epsilon >0$, and confidence parameter: $\delta >0$
   \STATE {\bfseries Initialize:} 
   \STATE ~~~ $S \leftarrow [n]$, $\epsilon_0 \leftarrow \frac{\epsilon}{8}$, and $\delta_0 \leftarrow \frac{\delta}{2}$  
   \STATE ~~~ Divide $S$ into $G: = \lceil \frac{n}{k} \rceil$ sets $\cG_1, \cG_2, \cdots \cG_G$ such that $\cup_{j = 1}^{G}\cG_j = S$ and $\cG_{j} \cap \cG_{j'} = \emptyset, ~\forall j,j' \in [G], \, |G_j| = k,\, \forall j \in [G-1]$.
    \textbf{If} $|\cG_{G}| < k$, \textbf{then} set $\cR_1 \leftarrow \cG_G$  and $G = G-1$.
   \WHILE{$\ell = 1,2, \ldots$}
   \STATE Set $S \leftarrow \emptyset$, $\delta_\ell \leftarrow \frac{\delta_{\ell-1}}{2}, \epsilon_\ell \leftarrow \frac{3}{4}\epsilon_{\ell-1}$
   \FOR {$g = 1,2, \cdots G$}
    \STATE Initialize pairwise (empirical) win-count $w_{ij} \leftarrow 0$, for each item pair $i,j \in \cG_g$
	\FOR {$\tau = 1, 2, \ldots t\,\,(:= \frac{4k}{m\epsilon_\ell^2}\ln \frac{2k}{\delta_\ell})$}
   	\STATE Play the set $\cG_g$ (one round of battle)
   	\STATE Receive feedback: The top-$m$ ranking $\bsigma \in \bSigma_{\cG^\tau_{gm}}$, where $\cG^\tau_{gm} \subseteq \cG_g$, $|\cG^\tau_{gm}| = m$ 
   	\STATE Update win-count $w_{ij}$ of each item pair $i,j \in \cG_g$ using \algupdt$(\cG_g,\bsigma)$
   	\ENDFOR 
	\STATE Define $\hat p_{i,j} = \frac{w_{ij}}{w_{ij}+w_{ji}}, \, \forall i,j \in \cG_g$
   	\STATE If $\exists i \in \cG_g$ such that $\hp_{i j} + \frac{\epsilon_\ell}{2} \ge \frac{1}{2}, \, \forall j \in \cG_g$, then set $c_g \leftarrow i$, else select $c_g \leftarrow$ uniformly at random from $\cG_g$, and set $S \leftarrow S \cup \{c_g\}$
   \ENDFOR
   \STATE $S \leftarrow S \cup \cR_\ell$
   \IF{$(|S| == 1)$}
   \STATE Break (go out of the while loop)
   \ELSIF{$|S|\le k$}
   \STATE $S' \leftarrow $ Randomly sample $k-|S|$ items from $[n] \setminus S$, and $S \leftarrow S \cup S'$, $\epsilon_\ell \leftarrow \frac{2\epsilon}{3}$, $\delta_\ell \leftarrow {\delta}$  
  \ELSE
   \STATE Divide $S$ into $G: = \big \lceil \frac{|S|}{k} \big \rceil$ sets $\cG_1, \cdots \cG_G$ such that $\cup_{j = 1}^{G}\cG_j = S$, $\cG_{j} \cap \cG_{j'} = \emptyset, ~\forall j,j' \in [G], \, |G_j| = k,\, \forall j \in [G-1]$. \textbf{If} $|\cG_{G}| < k$, \textbf{then} set $\cR_{\ell+1} \leftarrow \cG_G$  and $G = G-1$.
   \ENDIF
   \ENDWHILE
   \STATE {\bfseries Output:} $r_*$ as the $\epsilon$-optimal item, where $S = \{r_*\}$ (i.e. $r_*$ is the only item remaining in $S$)
\end{algorithmic}
\end{algorithm}
\vspace{-10pt}
\end{center}

\vspace*{-10pt}
\section{Conclusion and Future Directions}
\label{sec:conclusion}

We have developed foundations for probably approximately correct (PAC) online learning with subset choices: introducing the problem of \emph{Battling-Bandits (BB)} with subset choice models -- a novel generalization of the well-studied Dueling-Bandit problem, where the objective is to find the `best item' by successively choosing subsets of $k$ alternatives from $n$ 
items, and subsequently receiving a set-wise feedback information in an online fashion. We have specifically studied the  Plackett-Luce (PL) choice model along with winner information {(WI)} and top ranking {(TR)} feedback, with the goal of finding an \emph{$(\epsilon,\delta)$-{PAC} item}: an $\epsilon$-approximation of the best item with probability at least $(1-\delta)$. Our results show that with just the {WI} feedback, playing a battling game is just as good as that of a dueling game $(k=2)$, as in this case  the required sample complexity of the PAC learning problem is independent of the subset set $k$. However with {TR} feedback, the battling framework provides a $\frac{1}{m}$-times faster learning rate, leading to an improved performance guarantee owing to the information gain with top-$m$ ranking feedback, as intuitively well justified as well.

\vspace{5pt}
\noindent
\textbf{Future Directions.} Our proposed framework of {\it Battling Bandits} opens up a set of new directions to pursue - with different feedback mechanisms, choice models (e.g. Multinomial Probit, Mallows, nested logit, generalized extreme-value models etc.), other learning objectives, etc. It is an interesting open problem to analyse the trade-off between the subset size $k$ and the learning rate for other choice models with different feedback mechanisms. Another relevant direction to pursue within battling bandits could be to extend it to more general settings such as revenue maximization \citep{Agrawal+16}, learning with cost budgets \citep{BMAB16,BMAB17}, feature-based preference information and adversarial choice feedback \citep{Adv_DB}.

%% file: appendix.tex
\newpage
\onecolumn

\appendix
{
\section*{\centering \Large{Supplementary for PAC Battling Bandits in the Plackett-Luce Model}}
}
  

\section{Appendix for Section \ref{sec:iia}}
\label{app:iia}

\subsection{Proof of Lemma \ref{lem:pl_simulator}}

\plsimulator*

\begin{proof}
	We prove the lemma by using a coupling argument. Consider the following `simulator' or probability space for the Plackett-Luce choice model that specifically depends on the item pair $i,j$, constructed as follows. Let $Z_1, Z_2, \ldots$ be a sequence of iid Bernoulli random variables with success parameter $\theta_i/(\theta_i + \theta_j)$. A counter is first initialized to $0$. At each time $t$, given $S_1, i_1, \ldots, S_{t-1}, i_{t-1}$ and $S_t$,  an independent coin is tossed with probability of heads $(\theta_i + \theta_j)/\sum_{k \in S_t} \theta_k$. If the coin lands tails, then $i_t$ is drawn as an independent sample from the Plackett-Luce distribution over $S_t \setminus \{i,j\}$, else, the counter is incremented by $1$, and $i_t$ is returned as $i$ if $Z_C = 1$ or $j$ if $Z_C = 0$ where $C$ is the present value of the counter.
	
	It may be checked that the construction above indeed yields the correct joint distribution for the sequence $i_1, S_1, \ldots, i_T, S_T$ as desired, due to the independence of irrelevant alternatives (IIA) property of the Plackett-Luce choice model: 
	\[ Pr(i_t = i | i_t \in \{i,j\}, S_t) = \frac{Pr(i_t = i | S_t)}{Pr(i_t \in \{i,j\} | S_t)} = \frac{\theta_i/\sum_{k \in S_t} \theta_k}{(\theta_i + \theta_j)/\sum_{k \in S_t} \theta_k} = \frac{\theta_i}{\theta_i + \theta_j}. \]
	Furthermore, $i_t \in \{i,j\}$ if and only if $C$ is incremented at round $t$, and $i_t = i$ if and only if $C$ is incremented at round $t$ and $Z_C = 1$. We thus have
\begin{align*}
	Pr &\left( \frac{n_i(T)}{n_{ij}(T)} - \frac{\theta_i}{\theta_i + \theta_j} \ge \eta, \; n_{ij}(T) \geq v \right) = Pr\left(\frac{\sum_{\ell=1}^{n_{ij}(T)} Z_\ell}{n_{ij}(T)} - \frac{\theta_i}{\theta_i + \theta_j}\ge\eta, \; n_{ij}(T) \geq v \right) \\
	&= \sum_{m = v}^T Pr\left(\frac{\sum_{\ell=1}^{n_{ij}(T)} Z_\ell}{n_{ij}(T)} - \frac{\theta_i}{\theta_i + \theta_j}\ge\eta, \; n_{ij}(T) = m \right) \\
	&= \sum_{m = v}^T Pr\left(\frac{\sum_{\ell=1}^{m} Z_\ell}{m} - \frac{\theta_i}{\theta_i + \theta_j}\ge\eta, \; n_{ij}(T) = m \right) \\
	&\stackrel{(a)}{=} \sum_{m = v}^T Pr\left(\frac{\sum_{\ell=1}^{m} Z_\ell}{m} - \frac{\theta_i}{\theta_i + \theta_j}\ge\eta \right) \, Pr \left( n_{ij}(T) = m \right) \\
	&\stackrel{(b)}{\leq} \sum_{m = v}^T Pr \left( n_{ij}(T) = m \right) \, e^{-2m \eta^2} \leq e^{-2v \eta^2},
\end{align*}	
where $(a)$ uses the fact that $S_1, \dots, S_T, X_1, \ldots, X_T$ are independent of $Z_1, Z_2, \ldots,$, and so $n_{ij}(T) \in \sigma(S_1, \dots, S_T, X_1, \ldots, X_T)$ is independent of $Z_1, \ldots, Z_m$ for any fixed $m$, and $(b)$ uses Hoeffding's concentration inequality for the iid sequence $Z_i$. 

Similarly, one can also derive
\[
Pr\left( \frac{n_i(T)}{n_{ij}(T)} - \frac{\theta_i}{\theta_i + \theta_j} \le -\eta, \; n_{ij}(T) \geq v \right) \le e^{-2v \eta^2}, 
\]
which concludes the proof.
\end{proof}


\section{Appendix for Section \ref{sec:res_wi}}
\label{app:reswi}
  
\subsection{Proof of Theorem \ref{thm:lb_plpac_win}}
\label{app:wilb}

\lbwin*

\begin{proof}
We will apply Lemma \ref{lem:gar16} to derive the desired lower bounds of Theorem \ref{thm:lb_plpac_win} for BB-PL with {WI} feedback model. 

Let us consider a bandit instance with the arm set containing all subsets of size $k$: $\cA = \{S = (S(1), \ldots, S(k)) \subseteq [n] ~|~ S(i) < S(j), \, \forall i < j\}$. 
Let $\bnu^1$ be the true distribution associated with the bandit arms, given by the Plackett-Luce parameters:
\begin{align*}
\textbf{True Instance} ~(\bnu^1): \theta_j^1 = \theta\bigg( \frac{1}{2} - \epsilon\bigg), \forall j \in [n]\setminus \{1\}, \text{ and } \theta_1^1 = \theta\bigg( \frac{1}{2} + \epsilon \bigg),
\end{align*}

for some $\theta \in \R_+, ~\epsilon > 0$. Now for every suboptimal item $a \in [n]\setminus \{1\}$, consider the modified instances $\bnu^a$ such that:
\begin{align*}
\textbf{Instance--a} ~(\bnu^a): \theta^a_j = \theta\bigg( \frac{1}{2} - \epsilon \bigg)^2, \forall j \in [n]\setminus \{a,1\}, \, \theta_1^a = \theta\bigg( \frac{1}{4} - \epsilon^2 \bigg), \text{ and } \theta_a^a = \theta\bigg( \frac{1}{2} + \epsilon \bigg)^2.
\end{align*}

For problem instance $\bnu^a, \, a \in [n]\setminus\{1\}$, the probability distribution associated with arm $S \in \cA$ is given by
\[
\nu^a_S \sim Categorical(p_1, p_2, \ldots, p_k), \text{ where } p_i = Pr(i|S), ~~\forall i \in [k], \, \forall S \in \cA,
\]
where $Pr(i|S)$ is as defined in Section \ref{sec:feed_mod}. 
Note that the only $\epsilon$-optimal arm for \textbf{Instance-a} is arm $a$. Now applying Lemma \ref{lem:gar16}, for some event $\cE \in \cF_\tau$ we get,

\begin{align}
\label{eq:FI_a}
\sum_{\{S \in \cA : a \in S\}}\E_{\bnu^1}[N_S(\tau_A)]KL(\bnu^1_S, \bnu^a_S) \ge {kl(Pr_{\nu}(\cE),Pr_{\nu'}(\cE))}.
\end{align}

The above result holds from the straightforward observation that for any arm $S \in \cA$ with $a \notin S$, $\bnu^1_S$ is same as $\bnu^a_S$, hence $KL(\bnu^1_S, \bnu^a_S)=0$, $\forall S \in \cA, \,a \notin S$. 
For notational convenience, we will henceforth denote $S^a = \{S \in \cA : a \in S\}$. 

Now let us analyse the right hand side of \eqref{eq:FI_a}, for any set $S \in S^a$. We further denote $r = \1(1 \in S)$, $q = (k-1-r)$, and $R = \frac{\frac{1}{2}+\epsilon}{\frac{1}{2}-\epsilon}$.
Note that
\begin{align*}
\nu^1_S(i) = 
\begin{cases} 
\frac{\theta(\frac{1}{2}+\epsilon)}{r\theta(\frac{1}{2}+\epsilon) + (k-r)\theta(\frac{1}{2}-\epsilon)} = \frac{R}{rR + (k-r)}, \forall i \in [k], \text{ such that } S(i) = 1,\\
\frac{\theta(\frac{1}{2}-\epsilon)}{r\theta(\frac{1}{2}+\epsilon) + (k-r)\theta(\frac{1}{2}-\epsilon)} = \frac{1}{rR + (k-r)}, \text{ otherwise. }
\end{cases}
\end{align*}

On the other hand, for problem \textbf{Instance-a}, we have that: 

\begin{align*}
\nu^a_S(i) = 
\begin{cases} 
\frac{R}{rR + R^2 + q}, \forall i \in [k], \text{ such that } S(i) = 1,\\
\frac{R^2}{rR + R^2 + q}, \forall i \in [k], \text{ such that } S(i) = a,\\
\frac{1}{rR + R^2 + q}, \text{ otherwise. }
\end{cases}
\end{align*}

Now using the following upper bound on $KL(\p_1,\p_2) \le \sum_{x \in \X}\frac{p_1^2(x)}{p_2(x)} -1$, $\p_1$ and $\p_2$ be two probability mass functions on the discrete random variable $\X$ \citep{klub16} we get:

\begin{align*}
KL(\bnu^1_S, \bnu^a_S) & \le \frac{rR + R^2 + q}{(rR + k-r)^2} (rR + \frac{1}{R} + q)-1.
\end{align*}

Replacing $q$ by $(k-1-r)$ and re-arranging terms, we get

 \begin{align}
 \label{eq:win_lb0}
\nonumber  KL(\bnu^1_S, \bnu^a_S) & \le \frac{(rR + (k-r) + (R^2-1))(rR + (k-r) + (R^{-2}-1))}{(rR + k-r)^2}-1 \\ 
 & = \frac{(rR + k-r-1)}{(rR + k-r)^2}\Big( R - \frac{1}{R}\Big)^2 \le \frac{1}{k}\Big( R - \frac{1}{R}\Big)^2 ~~[\text{since } s \ge 0, \text{ and } R > 1].
 \end{align}

Note that the only $\epsilon$-optimal arm for any \textbf{Instance-a} is arm $a$, for all $a \in [n]$.
Now, consider $\cE_0 \in \cF_\tau$ be an event such that the algorithm $A$ returns the element $i = 1$, and let us analyse the left hand side of \eqref{eq:FI_a} for $\cE = \cE_0$. Clearly, $A$ being an $(\epsilon,\delta)$-PAC algorithm, we have $Pr_{\bnu^1}(\cE_0) > 1-\delta$, and $Pr_{\bnu^a}(\cE_0) < \delta$, for any suboptimal arm $a \in [n]\setminus\{1\}$. Then we have 

\begin{align}
\label{eq:win_lb2}
kl(Pr_{\bnu^1}(\cE_0),Pr_{\bnu^a}(\cE_0)) \ge kl(1-\delta,\delta) \ge \ln \frac{1}{2.4\delta}
\end{align}

where the last inequality follows from \citet[Equation $(3)$]{Kaufmann+16_OnComplexity}.

Now applying \eqref{eq:FI_a} for each modified bandit \textbf{Instance-$\bnu^a$}, and summing over all suboptimal items $a \in [n]\setminus \{1\}$ we get,

\begin{align}
\label{eq:win_lb2.5}
\sum_{a = 2}^{n}\sum_{\{S \in \cA \mid a \in S\}}\E_{\bnu^1}[N_S(\tau_A)]KL(\bnu^1_S,\bnu^a_S) \ge (n-1)\ln \frac{1}{2.4\delta}.
\end{align}

Moreover, using \eqref{eq:win_lb0}, the term of the right hand side of \eqref{eq:win_lb2.5} can be further upper bounded as

\begin{align}
\label{eq:win_lb3}
\nonumber \sum_{a = 2}^{n}&\sum_{\{S \in \cA \mid a \in S\}} \E_{\bnu^1}[N_S(\tau_A)]KL(\bnu^1_S,\bnu^a_S) \le \sum_{S \in \cA}\E_{\bnu^1}[N_S(\tau_A)]\sum_{\{a \in S \mid a \neq 1\}}\frac{1}{k}\Bigg( R - \frac{1}{R} \Bigg)^2\\
\nonumber  & = \sum_{S \in \cA}\E_{\bnu^1}[N_S(\tau_A)]\frac{k - \big(\1(1 \in S)\big)}{k}\Bigg( R - \frac{1}{R} \Bigg)^2\\
& \le \sum_{S \in \cA}\E_{\bnu^1}[N_S(\tau_A)](256\epsilon^2) ~~\Bigg[\text{since } \Bigg( R - \frac{1}{R} \Bigg) = \frac{8\epsilon}{(1-4\epsilon^2)} \le 16\epsilon, \forall \epsilon \in [0,\frac{1}{\sqrt{8}}] \Bigg].
\end{align}

Finally noting that $\tau_A = \sum_{S \in \cA}[N_S(\tau_A)]$, combining \eqref{eq:win_lb3} and \eqref{eq:win_lb2.5}, we get 

\begin{align*}
(256\epsilon^2)\E_{\bnu^1}[\tau_A] =  \sum_{S \in \cA}\E_{\bnu^1}[N_S(\tau_A)](256\epsilon^2) \ge (n-1)\ln \frac{1}{2.4\delta}.
\end{align*}
Thus above construction shows the existence of a problem instance $\bnu = \bnu^1$, such that $\E_{\bnu^1}[\tau_A] = \Omega(\frac{n}{\epsilon^2}\ln \frac{1}{2.4\delta})$, which concludes the proof.

\end{proof}  

\subsection{Proof of Theorem \ref{thm:trace_best}}
\label{app:wiub_trc}

\ubtrc*

\begin{proof}
We start by analyzing the required sample complexity first. Note that the `while loop' of Algorithm \ref{alg:trace_bb} always discards away $k-1$ items per iteration. Thus, $n$ being the total number of items the loop can be executed is at most for $\lceil \frac{n}{k-1} \rceil $ many number of iterations. Clearly, the sample complexity of each iteration being $t = \frac{2k}{\epsilon^2}\ln \frac{n}{2\delta}$, the total sample complexity of the algorithm thus becomes $\big( \lceil \frac{n}{k-1} \rceil \big)\frac{2k}{\epsilon^2}\ln \frac{n}{2\delta} \le \big(  \frac{n}{k-1} + 1 \big)\frac{2k}{\epsilon^2}\ln \frac{n}{2\delta} = \big( n + \frac{n}{k-1} + k \big)\frac{2}{\epsilon^2}\ln \frac{n}{2\delta} = O(\frac{n}{\epsilon^2}\ln \frac{n}{\delta})$.

We now prove the $(\epsilon,\delta)$-{PAC} correctness of the algorithm. As argued before, the `while loop' of Algorithm \ref{alg:trace_bb} can run for maximum $\lceil \frac{n}{k-1} \rceil $ many number of iterations. We denote the iterations by $\ell = 1,2, \ldots, \lceil \frac{n}{k-1} \rceil $, and the corresponding set $\cA$ of iteration $\ell$ by $\cA_\ell$. 

Note that our idea is to retain the estimated best item in `running winner' $r_\ell$ and compare it with the `empirical best item' $c_\ell$ of $\cA_\ell$ at every iteration $\ell$. The crucial observation lies in noting that at any iteration $\ell$, $r_\ell$ gets updated as follows:

\lemtrc*

\begin{proof}
Consider any set $\cA_\ell$, by which we mean the state of $\cA$ in the algorithm at iteration $\ell$. The crucial observation to make is that since $c_\ell$ is the empirical winner of $t$ rounds of battle, then $w_{c_\ell} \ge \frac{t}{k}$. Thus $w_{c_\ell} + w_{r_\ell} \ge \frac{t}{k}$. Let $n_{ij}:= w_i + w_j$ denotes the total number of pairwise comparisons between item $i$ and $j$ in $t$ rounds, for any $i,j \in \cA_\ell$. Then clearly, $0 \le n_{ij} \le t$ and $n_{ij} = n_{ji}$. Specifically we have $\hp_{r_\ell c_\ell} = \frac{w_{r_\ell}}{w_{r_\ell}+w_{c_\ell}} = \frac{w_{r_\ell}}{n_{r_\ell c_\ell}}$.
We prove the claim by analyzing the following cases: 

\textbf{Case 1.} (If $p_{c_\ell r_\ell} \le \frac{1}{2}$, \algtrc\, retains $r_{\ell+1} \leftarrow r_\ell$): Note that \algtrc\, replaces $r_{\ell+1}$ by $c_\ell$ only if $\hp_{c_\ell,r_\ell} > \frac{1}{2} + \frac{\epsilon}{2} $, but this happens with probability: 

\begin{align*}
& Pr\Bigg( \bigg\{ \hp_{c_\ell r_\ell} > \frac{1}{2} + \frac{\epsilon}{2} \bigg\} \Bigg) \\
& = Pr\Bigg( \bigg\{ \hp_{c_\ell r_\ell} > \frac{1}{2} + \frac{\epsilon}{2} \bigg\} \cap \bigg\{ n_{c_\ell r_\ell} \ge \frac{t}{k} \bigg\}\Bigg) \hspace*{5pt}+ \\
& \hspace*{1.7in} \cancelto{0}{Pr\Bigg(\bigg\{ n_{c_\ell r_\ell} < \frac{t}{k} \bigg\}\Bigg)}Pr\Bigg( \bigg\{ \hp_{c_\ell r_\ell} > \frac{1}{2} + \frac{\epsilon}{2} \bigg\} \Big | \bigg\{ n_{c_\ell r_\ell} < \frac{t}{k} \bigg\}\Bigg)\\
& \le Pr\Bigg( \bigg\{ \hp_{c_\ell r_\ell} - p_{c_\ell r_\ell} > \frac{\epsilon}{2} \bigg\} \cap \bigg\{ n_{c_\ell r_\ell} \ge \frac{t}{k} \bigg\} \Bigg) \le \exp\Big( -2\dfrac{t}{k}\bigg(\frac{\epsilon}{2}\bigg)^2 \Big) = \frac{\delta}{2n},
\end{align*}
where the first inequality follows as $p_{c_\ell r_\ell} \le \frac{1}{2}$, and the second inequality is by applying Lemma \ref{lem:pl_simulator} with $\eta = \frac{\epsilon}{2}$ and $v = \frac{t}{k}$.
We now proceed to the second case:

\textbf{Case 2.} (If $p_{c_\ell r_\ell} \ge \frac{1}{2} + \epsilon$, \algtrc\,  sets $r_{\ell+1} \leftarrow c_\ell$): Recall again that \algtrc\, retains $r_{\ell+1}  \leftarrow r_\ell$ only if $\hp_{c_\ell,r_\ell} \le \frac{1}{2} + \frac{\epsilon}{2} $. This happens with probability:

\begin{align*}
& Pr\Bigg( \bigg\{ \hp_{c_\ell r_\ell} \le \frac{1}{2} + \frac{\epsilon}{2} \bigg\} \Bigg) \\
& = Pr\Bigg( \bigg\{ \hp_{c_\ell r_\ell} \le \frac{1}{2} + \frac{\epsilon}{2} \bigg\} \cap \bigg\{ n_{c_\ell r_\ell} \ge \frac{t}{k} \bigg\}\Bigg) \hspace*{5pt}+ \\
& \hspace*{1.7in} \cancelto{0}{Pr\Bigg(\bigg\{ n_{c_\ell r_\ell} < \frac{t}{k} \bigg\}\Bigg)}Pr\Bigg( \bigg\{ \hp_{c_\ell r_\ell} \le \frac{1}{2} + \frac{\epsilon}{2} \bigg\} \Big | \bigg\{ n_{c_\ell r_\ell} < \frac{t}{k} \bigg\}\Bigg)\\
& = Pr\Bigg( \bigg\{ \hp_{c_\ell r_\ell} \le \frac{1}{2} + \epsilon - \frac{\epsilon}{2} \bigg\} \cap \bigg\{ n_{c_\ell r_\ell} \ge \frac{t}{k} \bigg\} \Bigg) \\
& \le Pr\Bigg( \bigg\{ \hp_{c_\ell r_\ell} - p_{c_\ell r_\ell} \le - \frac{\epsilon}{2} \bigg\} \cap \bigg\{ n_{c_\ell r_\ell} \ge \frac{t}{k} \bigg\} \Bigg) \le \exp\Big( -2\dfrac{t}{k}\bigg(\frac{\epsilon}{2}\bigg)^2 \Big) = \frac{\delta}{2n},
\end{align*}
where the first inequality holds as $p_{c_\ell r_\ell} \ge \frac{1}{2} + \epsilon$, and the second one by applying Lemma \ref{lem:pl_simulator} with $\eta = \frac{\epsilon}{2}$ and $v = \frac{t}{k}$. The proof follows combining the above two cases.
\end{proof}

Given Algorithm \ref{alg:trace_bb} satisfies Lemma \ref{lem:c_vs_r}, and taking union bound over $(k-1)$ elements in $\cA_\ell \setminus\{r_\ell\}$, we get that with probability at least $(1-\frac{(k-1)\delta}{2n})$,

\begin{align}
\label{eq:r_vs_c}
p_{r_{\ell+1}r_\ell} \ge \frac{1}{2} \text{ and, } p_{r_{\ell+1}c_\ell} \ge \frac{1}{2} - \epsilon.
\end{align} 

Above suggests that for each iteration $\ell$, the estimated `best' item $r_\ell$ only gets improved as $p_{r_{\ell+1}r_\ell} \ge \frac{1}{2}$. Let, $\ell_*$ denotes the specific iteration such that $1 \in \cA_\ell$ for the first time, i.e. $\ell_* = \min\{ \ell \mid 1 \in \cA_\ell \}$. Clearly $\ell_* \le \lceil \frac{n}{k-1} \rceil$. 
Now \eqref{eq:r_vs_c} suggests that with probability at least $(1-\frac{(k-1)\delta}{2n})$, $p_{r_{\ell_*+1}1} \ge \frac{1}{2} - \epsilon$. Moreover \eqref{eq:r_vs_c} also suggests that for all $\ell > \ell_*$, with probability at least $(1-\frac{(k-1)\delta}{2n})$, $p_{r_{\ell+1}r_\ell} \ge \frac{1}{2}$, which implies for all $\ell > \ell_*$, $p_{r_{\ell+1}1} \ge \frac{1}{2} - \epsilon$ as well -- This holds due to the following transitivity property of the Plackett-Luce model: For any three items $i_1,i_2,i_3 \in [n]$, if $p_{i_1i_2}\ge \frac{1}{2}$ and $p_{i_2i_3}\ge \frac{1}{2}$, then we have $p_{i_1i_3}\ge \frac{1}{2}$ as well. 

This argument finally leads to $p_{r_*1} \ge \frac{1}{2} - \epsilon$. Since failure probability at each iteration $\ell$ is at most $\frac{(k-1)\delta}{2n}$, and Algorithm \ref{alg:trace_bb} runs for maximum $\lceil \frac{n}{k-1} \rceil$ many number of iterations, using union bound over $\ell$, the total failure probability of the algorithm is at most $\lceil \frac{n}{k-1} \rceil \frac{(k-1)\delta}{2n} \le (\frac{n}{k-1}+1)\frac{(k-1)\delta}{2n} = \delta\Big( \frac{n+k-1}{2n} \Big) \le \delta$ (since $k \le n$). This concludes the correctness of the algorithm showing that it indeed satisfies the $(\epsilon,\delta)$-PAC objective.
\end{proof}

\vspace*{-20pt}

\setcounter{algorithm}{1}

\begin{center}
\begin{algorithm}[H]
   \caption{\textbf{\algdiv}}
   \label{alg:div_bb}
\begin{algorithmic}[1]
   \STATE {\bfseries Input:} 
   \STATE ~~~ Set of items: $[n]$, Subset size: $n \geq k > 1$
   \STATE ~~~ Error bias: $\epsilon >0$, Confidence parameter: $\delta >0$
   \STATE {\bfseries Initialize:} 
   \STATE ~~~ $S \leftarrow [n]$, $\epsilon_0 \leftarrow \frac{\epsilon}{8}$, and $\delta_0 \leftarrow \frac{\delta}{2}$  
   \STATE ~~~ Divide $S$ into $G: = \lceil \frac{n}{k} \rceil$ sets $\cG_1, \cG_2, \ldots, \cG_G$ such that $\cup_{j = 1}^{G}\cG_j = S$ and $\cG_{j} \cap \cG_{j'} = \emptyset, ~\forall j,j' \in [G]$, where $|G_j| = k,\, \forall j \in [G-1]$
   \STATE ~~~ \textbf{If} $|\cG_{G}| < k$, \textbf{then} set $\cR_1 \leftarrow \cG_G$  and $G = G-1$
   \WHILE{$\ell = 1,2, \ldots$}
   \STATE Set $S \leftarrow \emptyset$, $\delta_\ell \leftarrow \frac{\delta_{\ell-1}}{2}, \epsilon_\ell \leftarrow \frac{3}{4}\epsilon_{\ell-1}$
   \FOR {$g = 1,2, \ldots, G$}
   \STATE Play the set $\cG_g$ for $t:= \frac{k}{2\epsilon_\ell^2}\ln \frac{k}{\delta_\ell}$ rounds
   \STATE $w_i \leftarrow$ Number of times $i$ won in $t$ plays of $\cG_g$, $\forall i \in \cG_g$
   \STATE Set $c_g \leftarrow \underset{i \in \cA}{{\arg\max}}~w_i$ and $S \leftarrow S \cup \{c_g\}$
   \ENDFOR
   \STATE $S \leftarrow S \cup \cR_\ell$
      \IF{$(|S| == 1)$}
   \STATE Break (go out of the while loop)
   \ELSIF{$|S|\le k$}
   \STATE $S' \leftarrow $ Randomly sample $k-|S|$ items from $[n] \setminus S$, and $S \leftarrow S \cup S'$, $\epsilon_\ell \leftarrow \frac{2\epsilon}{3}$, $\delta_\ell \leftarrow {\delta}$  
   \ELSE
   \STATE Divide $S$ into $G: = \lceil \frac{|S|}{k} \rceil$ sets $\cG_1, \cG_2, \ldots, \cG_G$, such that $\cup_{j = 1}^{G}\cG_j = S$, and $\cG_{j} \cap \cG_{j'} = \emptyset, ~\forall j,j' \in [G]$, where $|G_j| = k,\, \forall j \in [G-1]$
   \STATE \textbf{If} $|\cG_{G}| < k$, \textbf{then} set $\cR_{\ell+1} \leftarrow \cG_G$  and $G = G-1$
   \ENDIF
   \ENDWHILE
   \STATE {\bfseries Output:} $r_*$ as the $\epsilon$-optimal item, where $S = \{r^*\}$
\end{algorithmic}
\end{algorithm}
\vspace{-2pt}
\end{center}

\vspace{-15pt}

\subsection{Proof of Theorem \ref{thm:div_bb}}
\label{app:wiub_div}

\ubdiv*

\begin{proof}
For the notational convenience we will use $\tp_{ij} = p_{ij} - \frac{1}{2}, \, \forall i,j \in [n]$. We start by proving the following lemma which would be used crucially in the analysis:

\begin{lem}
\label{lem:pl_sst}
For any three items $a,b,c \in [n]$ such that $\theta_a > \theta_b > \theta_c$. If $\tp_{ba} > -\epsilon_1$, and $\tp_{cb} > -\epsilon_2$, where $\epsilon_1,\epsilon_2 > 0$, and $(\epsilon_1+\epsilon_2) < \frac{1}{2}$, then $\tp_{ca} > -(\epsilon_1+\epsilon_2)$.
\end{lem}

\begin{proof}
Note that $\tp_{ba} > -\epsilon_1 \implies \frac{\theta_b - \theta_a}{2(\theta_b+\theta_a)} > -\epsilon_1 \implies \frac{\theta_b}{\theta_a} > \frac{(1-2\epsilon_1)}{(1+2\epsilon_1)}$. 

Similarly we have $\tp_{cb} > -\epsilon_2 \implies \frac{\theta_c}{\theta_b} > \frac{(1-2\epsilon_2)}{(1+2\epsilon_2)}$. Combining above we get 

\begin{align*}
\frac{\theta_c}{\theta_a} > \frac{(1-2\epsilon_1)}{(1+2\epsilon_1)} \frac{(1-2\epsilon_2)}{(1+2\epsilon_2)} >  & \frac{1-2(\epsilon_1+\epsilon_2) + \epsilon_1\epsilon_2}{1+2(\epsilon_1+\epsilon_2) + \epsilon_1\epsilon_2} > \frac{1-2(\epsilon_1+\epsilon_2)}{1+2(\epsilon_1+\epsilon_2)}, ~\bigg[\text{since,} (\epsilon_1+\epsilon_2) < \frac{1}{2}  \bigg]\\
& \implies \tp_{ca} = \frac{\theta_c - \theta_a}{2(\theta_c+\theta_a)} > -(\epsilon_1+\epsilon_2),
\end{align*}
which concludes the proof.
\end{proof}

We now analyze the required sample complexity of \algdiv. 
For clarity of notations we will denote the set $S$ at iteration $\ell$ by $S_\ell$.
Note that at any iteration $\ell$, any set $\cG_g$ is played for exactly $t= \frac{k}{2\epsilon_\ell^2}\ln \frac{k}{\delta_\ell}$ many number of rounds. Also since the algorithm discards away exactly $k-1$ items from each set $\cG_g$, hence the maximum number of iterations possible is $\lceil  \ln_k n \rceil$. Now at any iteration $\ell$, since $G = \Big\lfloor \frac{|S_\ell|}{k} \Big\rfloor < \frac{|S_\ell|}{k}$, the total sample complexity for iteration $\ell$ is at most $\frac{|S_\ell|}{k}t \le \frac{n}{2k^{\ell-1}\epsilon_\ell^2}\ln \frac{k}{\delta_\ell}$, as $|S_\ell| \le \frac{n}{k^\ell}$ for all $\ell \in [\lfloor  \ln_k n \rfloor]$. Also note that for all but last iteration $\ell \in [\lfloor  \ln_k n \rfloor]$, $\epsilon_\ell = \frac{\epsilon}{8}\bigg( \frac{3}{4} \bigg)^{\ell-1}$, and $\delta_\ell = \frac{\delta}{2^{\ell+1}}$.
Moreover for the last iteration $\ell = \lceil  \ln_k n \rceil$, the sample complexity is clearly $t= \frac{2k}{\epsilon^2}\ln \frac{2k}{\delta}$, as in this case $\epsilon_\ell = \frac{\epsilon}{2}$, and $\delta_\ell = \frac{\delta}{2}$, and $|S| = k$.
Thus the total sample complexity of Algorithm \ref{alg:div_bb} is given by 

\begin{align*}
\sum_{\ell = 1}^{\lceil  \ln_k n \rceil}\frac{|S_\ell|}{2\epsilon_\ell^2}\ln \frac{k}{\delta_\ell}  
& \le \sum_{\ell = 1}^{\infty}\frac{n}{2k^\ell\bigg(\frac{\epsilon}{8}\big(\frac{3}{4}\big)^{\ell-1}\bigg)^2}k\ln \frac{k 2^{\ell+1}}{\delta} + \frac{2k}{\epsilon^2}\ln \frac{2k}{\delta}\\ 
& \le \frac{64n}{2\epsilon^2}\sum_{\ell = 1}^{\infty}\frac{16^{\ell-1}}{(9k)^{\ell-1}}\Big( \ln \frac{k}{\delta} + {(\ell+1)} \Big) + \frac{2k}{\epsilon^2}\ln \frac{2k}{\delta} \\
& \le \frac{32n}{\epsilon^2}\ln \frac{k}{\delta}\sum_{\ell = 1}^{\infty}\frac{4^{\ell-1}}{(9k)^{\ell-1}}\Big( {3\ell} \Big) + \frac{2k}{\epsilon^2}\ln \frac{2k}{\delta} = O\bigg(\frac{n}{\epsilon^2}\ln \frac{k}{\delta}\bigg) ~[\text{for any } k > 1].\end{align*}

Above proves the sample complexity bound of Theorem \ref{thm:div_bb}. We next prove the $(\epsilon,\delta)$-{PAC} property of \algdiv. 
The crucial observation lies in the fact that, at any iteration $\ell$, for any set $\cG_g$ ($g = 1,2,\ldots, G$), the item $c_g$ retained by the algorithm is likely to be not more than $\epsilon_\ell$-worse than the best item (the one with maximum score parameter $\theta$) of the set $\cG_g$, with probability at least $(1-\delta_\ell)$. More precisely, we claim the following:

\lemdiv*

\begin{proof}
Let us define $\hp_{ij} = \frac{w_i}{w_i + w_j}, \, \forall i,j \in \cG_g, i \neq j$. Then clearly $\hp_{c_gi_g} \ge \frac{1}{2}$, as 
$c_g$ is the empirical winner in $t$ rounds, i.e. $c_g \leftarrow \underset{i \in \cG_g}{{\arg\max}}~w_i$. Moreover $c_g$ being the empirical winner of $\cG_g$ we also have $w_{c_g} \ge \frac{t}{k}$, and thus $w_{c_\ell} + w_{r_\ell} \ge \frac{t}{k}$ as well. Let $n_{ij}:= w_i + w_j$ denotes the number of pairwise comparisons of item $i$ and $j$ in $t$ rounds, $i,j \in \cG_g$. Clearly $0 \le n_{ij} \le t$. Then let us analyze the probability of a `bad event' where $c_g$ is indeed such that $p_{c_gi_g} < \frac{1}{2} - \epsilon_\ell$ but we have $c_g$ beating $i_g$ empirically:

\begin{align*}
& Pr\Bigg( \bigg\{ \hp_{c_gi_g} \ge \frac{1}{2} \bigg\} \Bigg)\\
& = Pr\Bigg( \bigg\{ \hp_{c_gi_g} \ge \frac{1}{2} \bigg\} \hspace*{-2pt} \cap \hspace*{-2pt} \bigg\{ n_{c_gi_g} \ge \frac{t}{k} \bigg\} \Bigg) \hspace*{-3pt}+ \hspace*{-3pt}\cancelto{0}{Pr\Bigg(\bigg\{ n_{c_gi_g} < \frac{t}{k} \bigg\}\Bigg)}Pr\Bigg( \bigg\{ \hp_{c_gi_g} \ge \frac{1}{2} \bigg\} \Big | \bigg\{ n_{c_gi_g} \hspace*{-2pt}<\hspace*{-2pt} \frac{t}{k} \bigg\}\Bigg)\\
& =  Pr\Bigg( \bigg\{ \hp_{c_gi_g} - \epsilon_\ell \ge \frac{1}{2} - \epsilon_\ell \bigg\} \cap \bigg\{ n_{c_gi_g} \ge \frac{t}{k} \bigg\} \Bigg)\\
& \le Pr\Bigg( \bigg\{ \hp_{c_gi_g} - p_{c_g i_g} \ge {\epsilon_\ell} \bigg\} \cap \bigg\{ n_{c_gi_g} \ge \frac{t}{k} \bigg\} \Bigg)\\
& \le \exp\Big( -2\dfrac{t}{k}\big({\epsilon_\ell}\big)^2 \Big) = \frac{\delta_\ell}{k}.
\end{align*}

where the first inequality holds as $p_{c_gi_g} < \frac{1}{2} - \epsilon_\ell$, and the second inequality is by applying Lemma \ref{lem:pl_simulator} with $\eta = \epsilon_\ell$ and $v = \frac{t}{k}$.
Now taking union bound over all $\epsilon_\ell$-suboptimal elements $i'$ of $\cG_g$ (i.e. $p_{i'i_g} < \frac{1}{2} - \epsilon_\ell$), we get: 
\[Pr\Bigg(\bigg\{\exists i' \in \cG_g \mid p_{i'i_g} \hspace*{-2pt} < \hspace*{-2pt} \frac{1}{2} - \epsilon_\ell, \text{and } c_g = i' \bigg\}\Bigg) \hspace*{-1pt}\le\hspace*{-1pt} \frac{\delta_\ell}{k}\Bigg \vert\bigg\{\exists i' \in \cG_g \mid p_{i'i_g} \hspace*{-2pt} < \hspace*{-2pt} \frac{1}{2} - \epsilon_\ell, \text{and } c_g = i' \bigg\}\Bigg\vert \le \delta_\ell, 
\]
as $|\cG_g| = k$, and the claim follows henceforth.
\end{proof}

\begin{rem}
For the last iteration $\ell = \lceil \ln_k n \rceil$, since $\epsilon_\ell = \frac{\epsilon}{2}$, and $\delta_\ell = \frac{\delta}{2}$, applying Lemma \ref{lem:div_bb} on $S$, we get that $Pr\Bigg( p_{r_*i_g} < \frac{1}{2} - \frac{\epsilon}{2} \Bigg) \leq \frac{\delta}{2}$.  
\end{rem}

Now for each iteration $\ell$, let us define $g_\ell \in [G]$ to be the set that contains \textit{best item} of the entire set $S$, i.e. $\arg\max_{i \in S}\theta_i \in \cG_{g_\ell}$. Then applying Lemma \ref{lem:div_bb}, with probability at least $(1-\delta_\ell)$,\, $\tp_{c_{g_\ell}i_{g_\ell}} > -\epsilon_\ell$. Then, for each iteration $\ell$, applying Lemma \ref{lem:pl_sst} and Lemma \ref{lem:div_bb} to $\cG_{g_\ell}$, we finally get $\tp_{r_*1} > -\Big( \frac{\epsilon}{8} + \frac{\epsilon}{8}\Big(\frac{3}{4}\Big) + \cdots + \frac{\epsilon}{8}\big(\frac{3}{4}\big)^{\lfloor \ln_k n \rfloor} \Big) + \frac{\epsilon}{2} \ge -\frac{\epsilon}{8}\Big( \sum_{i = 0}^{\infty}\big( \frac{3}{4} \big)^{i} \Big) + \frac{\epsilon}{2} = \epsilon$.
(Note that, for above analysis to go through, it is in fact sufficient to consider only the set of iterations $\{\ell \ge \ell_0 \mid \ell_0 = \min\{l \mid 1 \notin \cR_l, \, l \ge 1\} \}$ because prior considering item $1$, it does not matter even if the algorithm mistakes in any of the iteration $\ell < \ell_0$). Thus assuming the algorithm does not fail in any of the iteration $\ell$, we finally have that $p_{r_*1} > \frac{1}{2} - \epsilon$.

Finally since at each iteration $\ell$, the algorithm fails with probability at most $\delta_\ell$, the total failure probability of the algorithm is at most $\Big( \frac{\delta}{4} + \frac{\delta}{8} + \cdots + \frac{\delta}{2^{\lceil \ln_k n  \rceil}} \Big) + \frac{\delta}{2} \le \delta$.
This concludes the correctness of the algorithm showing that it indeed satisfies the $(\epsilon,\delta)$-PAC objective.
\end{proof}

\subsection{Proof of Theorem \ref{thm:half_bb}}
\label{app:wiub_med}

\ubmed*

\begin{proof}
For clarity of notation, we will denote the set of remaining items $S$ at iteration $\ell$ by $S_\ell$.
We start by observing that at each iteration $\ell = 1,2,\ldots$, the size of the set of remaining items $S_{\ell+1}$ gets halved compared to that of the previous iteration $S_{\ell}$, since the algorithm discards away all the elements below the \emph{median item} $h_g$, as follows from the definition of median. This implies that the maximum number of iterations possible is $\ell = \lceil \ln n \rceil$, after which $|S| = 1$ and the algorithm returns $r_*$. 

We first analyze the sample complexity of the algorithm. Clearly each iteration $\ell$ uses a sample complexity of $t = \frac{k}{2\epsilon_\ell^2}\ln \frac{1}{\delta_\ell}$, and as argued before $\ell$ can be at most $\lceil \ln n \rceil$ which makes the total sample complexity of the algorithm:

\begin{align*}
\sum_{\ell = 1}^{\lceil  \ln n \rceil}\frac{|S_\ell|}{2\epsilon_\ell^2}\ln \frac{1}{\delta_\ell}  
& \le \sum_{\ell = 1}^{\infty}\frac{n}{2k^\ell\bigg(\frac{\epsilon}{4}\big(\frac{3}{4}\big)^{\ell-1}\bigg)^2}k\ln 4\Big( \frac{ 2^{\ell}}{\delta}\Big)
\le \frac{16n}{2\epsilon^2}\sum_{\ell = 1}^{\infty}\frac{16^{\ell-1}}{(9k)^{\ell-1}}\Big( \ln \frac{4}{\delta} + \ell \Big)\\
& \le \frac{8n}{\epsilon^2}\ln \frac{4}{\delta}\sum_{\ell = 1}^{\infty}\frac{16^{\ell-1}}{9k^{\ell-1}}\Big( {2\ell} \Big)= O\bigg(\frac{n}{\epsilon^2}\ln \frac{1}{\delta}\bigg) ~[\text{for any} k > 1].
\end{align*}

This ensures the sample complexity of Theorem \ref{thm:half_bb} holds good.

We are now only left with verifying the $(\epsilon,\delta)$-{PAC} property of the algorithm where lies the main difference of the analysis of \algmed\, from \algdiv. Consider any iteration $\ell \in [\lceil \ln n \rceil]$. The crucial observation is that, with high probability of at least $\big( 1- \delta_\ell \big)$ for any such $\ell$, and any set $\cG_g$ ($g = 1,2,\ldots, G$), some $\epsilon_\ell$-approximation of the \textit{`best-item'} (the one with the highest score parameter $\theta_i$) of $\cG_g$ must lie above the median in terms of the empirical win count $w_i$, and hence must be retained by the algorithm till the next iteration $\ell+1$. We prove this formally below. 

Our first claim starts by showing that for any set $\cG_g$, the empirical win count estimate $w_i$ of the best item $i_g := \arg \max_{i \in \cG_g}\theta_i$ (i.e. the one with highest score parameter $\theta_i$) can not be too small, as shown in Lemma \ref{lem:hv_n1}: 

\begin{lem}
\label{lem:hv_n1} 
Consider any particular set $\cG_g$ at any iteration $\ell \in \lceil  \ln n \rceil$. If $i_g := \arg \max_{i \in \cG_g}\theta_i$, then with probability at least $\Big(1-\frac{\delta_\ell}{4}\Big)$, the empirical win count $w_{i_g} > (1-\eta)\frac{t}{k}$, for any $\eta \in \big(\frac{3}{16},1 \big]$.
\end{lem} 

\begin{proof}
The proof follows from an straightforward application of Chernoff-Hoeffding's inequality \cite{CI_book}. 
Note that the algorithm plays each set $\cG_g$ for $t = \frac{k}{2\epsilon_\ell^2}\ln \frac{1}{\delta_\ell}$ number of times. Fix any iteration $\ell$ and a set $\cG_g$, $g \in 1,2, \ldots, G$. Suppose $i_\tau$ denotes the winner of $\tau$-{th} play of $\cG_g$, $\tau \in [t]$. Then clearly, for any item $i \in \cG_g$, $w_i = \sum_{\tau = 1}^{t}\1(i_\tau == i)$, where $\1(i_\tau == i)$ is a Bernoulli random variable with parameter $\frac{\theta_i}{\sum_{j \in \cG_g}\theta_j}$, by definition of {WI} feedback model. 
Also for $i = i_g$, we have $Pr(\{i_\tau = i_g\}) = \frac{\theta_{i_g}}{\sum_{j \in \cG_g}\theta_j} \ge \frac{1}{k}, \, \forall \tau \in [t]$, as follows from the definition $i_g := \arg \max_{i \in \cG_g}\theta_i$. Hence $\E[w_{i_g}] = \sum_{\tau = 1}^{t}\E[\1(i_\tau == i)] \ge \frac{t}{k}$. 
Now applying multiplicative Chernoff-Hoeffdings bound for $w_{i_g}$, we get that for any $\eta \in (\frac{3}{16},1]$, 

\begin{align*}
Pr\Big( w_{i_g} \le (1-\eta)\E[w_{i_g}] \Big) & \le \exp\bigg(- \frac{\E[w_{i_g}]\eta^2}{2}\bigg) \le \exp\bigg(- \frac{t\eta^2}{2k}\bigg) \\
& \le \exp\bigg(- \frac{\eta^2}{\epsilon_\ell^2} \ln \bigg( \frac{4}{\delta_\ell} \bigg) \bigg) \le \exp\bigg(- \ln \bigg( \frac{4}{\delta_\ell} \bigg) \bigg) = \frac{\delta_\ell}{4},
\end{align*}
where the second last inequality holds as $\eta > \frac{3}{16}$ and $\epsilon_\ell \le \frac{3}{16}$, for any iteration $\ell \in \lceil \ln n \rceil$; in other words for any $\eta \ge \frac{1}{4}$, we have $\frac{\eta}{\epsilon_\ell} > 1$ which leads to the second last inequality, and the proof follows henceforth.
\end{proof}

In particular, fixing $\eta = \frac{1}{2}$ in Lemma \ref{lem:hv_n1}, we get that with probability at least $\big(1-\frac{\delta_\ell}{4}\big)$,  $w_{i_g} > (1-\frac{1}{2})\E[w_{i_g}] > \frac{t}{2k}$. We now prove that for any set $\cG_g$, given its best item $i_g$ is selected as the winner for at least $\frac{t}{2k}$ times out of $t$ plays of $\cG_g$, the empirical estimate of $p_{i_gb}$, defined as $\hp_{i_gb} = \frac{w_{i_g}}{w_{i_g} + w_b}$, for any suboptimal element $b \in \cG_g$ (such that $p_{i_gb} > \frac{1}{2} + \epsilon$) can not be too misleading where empirical win count of $b$ exceeds that of $i_g$, i.e. $w_b > w_{i_g}$. The formal claim is as follows:

\lemmed*

\begin{proof}
First note since $w_{i_g} \ge \frac{t}{2k}$, this implies $w_{i_g} + w_{b} \ge \frac{t}{2k}$ as well. Let us define  $n_{ij} = w_i + w_j$ to be the number of pairwise comparisons of item $i$ and $j$ in $t$ rounds, for any $i,j \in \cG_g$, and $\hp_{ij} = \frac{w_i}{w_i + w_j}$ to be the empirical estimate of pairwise probability of item $i$ and $j$. Then,

\begin{align*}
Pr\Bigg( \bigg\{ w_b \ge w_{i_g} \bigg\} \cap \bigg\{ n_{i_gb} \ge \frac{t}{2k} \bigg\} \Bigg) & = Pr\Bigg( \bigg\{ \hp_{b i_g} \ge \frac{1}{2} \bigg\} \cap \bigg\{ n_{i_g b} \ge \frac{t}{2k} \bigg\} \Bigg)\\
& = Pr\Bigg( \bigg\{ \hp_{b i_g} - \epsilon_\ell \ge \frac{1}{2} - \epsilon_\ell \bigg\} \cap \bigg\{ n_{i_g b} \ge \frac{t}{2k} \bigg\} \Bigg)\\
& \le Pr\Bigg( \bigg\{ \hp_{b i_g} - p_{b i_g} \ge \epsilon_\ell \bigg\} \cap \bigg\{ n_{i_g b} \ge \frac{t}{2k} \bigg\} \Bigg)\\
& \le \exp\Big( -2\dfrac{t}{2k}\big({\epsilon_\ell}\big)^2 \Big) \le \frac{\delta_\ell}{4},
\end{align*}
where the second last inequality holds since $p_{b i_g} < \frac{1}{2} - \epsilon_\ell$. The last inequality follows by applying Lemma \ref{lem:pl_simulator} with $\eta = \epsilon_\ell$ and $v = \frac{t}{2k}$.

Using the results from  Lemma \ref{lem:hv_n1} we further get that for any such suboptimal element $b \in \cC_g$ with $p_{b i_g} < \frac{1}{2} - \epsilon_\ell$,

\begin{align*}
& Pr\big(\big\{ w_b \ge w_{i_g} \big\} \big) \\
& = Pr\Bigg( \bigg\{ w_b \ge w_{i_g} \bigg\} \cap \bigg\{ w_{i_g} < \frac{t}{2k} \bigg\} \Bigg) + Pr\Bigg( \bigg\{ w_b \ge w_{i_g} \bigg\} \cap \bigg\{ w_{i_g} \ge \frac{t}{2k} \bigg\} \Bigg)\\
& \le Pr\big(\big\{ w_{i_g} < \frac{t}{2k} \big\} \big) + Pr\Bigg( \bigg\{ w_b \ge w_{i_g} \bigg\} \cap \bigg\{ n_{i_gb} \ge \frac{t}{2k} \bigg\} \Bigg)\\
& \le \frac{\delta_\ell}{4} + \frac{\delta_\ell}{4} ~~\Big[ \text{Applying Lemma \ref{lem:hv_n1} with } \eta = \frac{1}{2}\Big] \le \frac{\delta_\ell}{2}
\end{align*}
\end{proof}

Now for any particular $\cG_g$, and for all suboptimal element $b \in \cG_g$, let us define an indicator random variable $F_b: = \1(w_b > w_{i_g})$. Note that by above claim we have $\E[F_b] = Pr(F_b)  = Pr(w_b > w_{i_g}) \le \frac{\delta_\ell}{2} $. Moreover if $\cB = \{b \in \cG_g \mid p_{b i_g} < \frac{1}{2} - \epsilon\}$ denotes the set of all $\epsilon_\ell$-suboptimal elements of $\cG_g$ (with respect to the best item $i_g$ of $\cG_g$), then clearly $|\cB| < |\cG_g|$, and thus we have $\E[\sum_{b \in \cB}F_b ] \le |\cG_g|\frac{\delta_\ell}{2}$. Now using Markov's inequality \cite{CI_book} we get: 

\begin{align*}
Pr\bigg[ \sum_{b \in \cB}F_b  \ge \frac{|\cG_g|}{2}  \bigg] \le \frac{\E[\sum_{b \in \cB}F_b ]}{|\cG_g|/2} \le \frac{|\cG_g| \delta_\ell/2}{|\cG_g|/2} = \delta_\ell.
\end{align*}
1
Above immediately implies that at any iteration $\ell$, and for any set $\cG_g$ in $\ell$, more than $\frac{|\cG_g|}{2}$ of the suboptimal elements of $\cG_g$ can not beat the best item $i_g$ in terms of empirical win count $w_i$. Thus there has to at least one non-suboptimal element $i' \in \cG_g$ ($i'$ could be $i_g$ itself), i.e. $p_{i_gi'} < \frac{1}{2} - \epsilon$, and $i'$ beats the median item $h_g$ with $w_{i'} \ge w_{h_g}$. Hence $i'$ would be retained by the algorithm in set $S$ till the next iteration $\ell+1$.

The above argument precisely shows that the best item of  the set $S$ at the beginning of iteration $\ell+1$, can not be $\epsilon_\ell$ worse than that of iteration $\ell$, for any $\ell = [\lceil \ln n \rceil]$. More formally,
if $i_\ell$ and $i_{\ell+1}$ respectively denote the best item of set $S$ at the beginning of iteration $\ell$ and $\ell+1$ respectively, i.e. $i_{\ell} := \arg\max_{i \in S_\ell}\theta_i$, and $i_{\ell+1} := \arg\max_{i \in S_{\ell+1}}\theta_i$,
then by Lemma \ref{lem:hv_pij}, with probability at least $(1-\delta_\ell)$, $p_{i_{\ell+1}i_\ell} > \frac{1}{2} -\epsilon_\ell$. Note that, at the beginning $i_{1} = 1$, which is the true best item (condorcet winner) $i^* = 1$ of $[n]$, as defined in Section \ref{sec:prb_setup}. Now applying Lemma \ref{lem:div_bb} and \ref{lem:pl_sst} for each iteration $\ell$, we get that the final item $r_*$ returned by the algorithm would satisfy $\tp_{r_*1} > -\Big( \frac{\epsilon}{4} + \frac{\epsilon}{4}\Big(\frac{3}{4}\Big) + \cdots + \frac{\epsilon}{4}\big(\frac{3}{4}\big)^{\lfloor \ln n \rfloor} \Big) \ge -\frac{\epsilon}{4}\Big( \sum_{i = 0}^{\infty}\big( \frac{3}{4} \big)^{i} \Big) = - \epsilon$.
Thus assuming the algorithm does not fail in any of the iteration $\ell$, we have that $p_{r_*1} > \frac{1}{2} - \epsilon$.

Finally at each iteration $\ell$, since the algorithm can fail with probability at most $\delta_\ell$, the total failure probability of the algorithm is at most $\Big( \frac{\delta}{2} + \frac{\delta}{4} + \cdots + \frac{\delta}{2^{\lceil \ln n  \rceil}} \Big)  \le \delta$. 
This concludes the proof as \algmed\, indeed satisfies the $(\epsilon,\delta)$-PAC objective.
\end{proof}


\section{Appendix for Section \ref{sec:res_fr}}
\label{app:resfr}

\subsection{Proof of Theorem \ref{thm:lb_pacpl_rnk}}
\label{app:frlb}

\lbrnk*

\begin{proof}
In this case too, we will use Lemma \ref{lem:gar16} to derive the desired lower bounds of Theorem \ref{thm:lb_plpac_win} for BB-PL with \textbf{TR} feedback model. 

Let us consider a bandit instance with the arm set containing all subsets of size $k$: $\cA = \{S = (S(1), \ldots, S(k)) \subseteq [n] ~|~ S(i) < S(j), \, \forall i < j\}$. 
Let $\bnu^1$ be the true distribution associated with the bandit arms, given by the Plackett-Luce parameters:
\begin{align*}
\textbf{True Instance} ~(\bnu^1): \theta_j^1 = \theta\bigg( \frac{1}{2} - \epsilon\bigg), \forall j \in [n]\setminus \{1\}, \text{ and } \theta_1^1 = \theta\bigg( \frac{1}{2} + \epsilon \bigg),
\end{align*}

for some $\theta \in \R_+, ~\epsilon > 0$. Now for every suboptimal item $a \in [n]\setminus \{1\}$, consider the modified instances $\bnu^a$ such that:
\begin{align*}
\textbf{Instance--a} ~(\bnu^a): \theta^a_j = \theta\bigg( \frac{1}{2} - \epsilon \bigg)^2, \forall j \in [n]\setminus \{a,1\}, \, \theta_1^a = \theta\bigg( \frac{1}{4} - \epsilon^2 \bigg), \text{ and } \theta_a^a = \theta\bigg( \frac{1}{2} + \epsilon \bigg)^2.
\end{align*}



It is now interesting to note that how top-$m$ ranking feedback affects the KL-divergence analysis, precisely the KL-divergence shoots up by a factor of $m$ which in fact triggers an $\frac{1}{m}$ reduction in regret learning rate.
Note that for top-$m$ ranking feedback for any problem \textbf{Instance-a} (for any $a \in [n]$), each $k$-set $S \subseteq [n]$ is associated to ${k \choose m} (m!)$ number of possible outcomes, each representing one possible ranking of set of $m$ items of $S$, say $S_m$. Also the probability of any permutation $\bsigma \in \bSigma_{S_m}$ is given by
$
\bnu^a_S(\bsigma) = Pr_{\btheta^a}(\bsigma|S),
$
where $Pr_{\btheta^a}(\bsigma|S)$ is as defined for top-$m$ (TR-$m$) ranking feedback (as in Sec. \ref{sec:feed_mod}). More formally,
for any problem \textbf{Instance-a}, we have that: 

\begin{align*}
\bnu^a_S(\bsigma) 
 = \prod_{i = 1}^{m}\frac{{\theta_{\sigma(i)}^a}}{\sum_{j = i}^{m}\theta_{\sigma(j)}^a + \sum_{j \in S \setminus \sigma(1:m)}\theta_{\sigma(j)}^a}, ~~~ \forall a \in [n],
\end{align*}

The important thing to note is that for any such top-$m$ ranking of $\bsigma \in \bSigma_S^m$, $KL(\bnu^1_S(\bsigma), \bnu^a_S(\bsigma)) = 0$ for any set $S \not\owns a$. Hence while comparing the KL-divergence of instances $\btheta^1$ vs $\btheta^a$, we need to focus only on sets containing $a$. Applying \emph{Chain-Rule} of KL-divergence, we now get

\begin{align}
\label{eq:lb_witf_kl}
\nonumber KL(\bnu^1_S, \bnu^a_S) = KL(\bnu^1_S(\sigma_1),& \bnu^a_S(\sigma_1)) + KL(\bnu^1_S(\sigma_2 \mid \sigma_1), \bnu^a_S(\sigma_2 \mid \sigma_1)) + \cdots \\ 
& + KL(\bnu^1_S(\sigma_m \mid \sigma(1:m-1)), \bnu^a_S(\sigma_m \mid \sigma(1:m-1))),
\end{align}
where we abbreviate $\sigma(i)$ as $\sigma_i$ and $KL( P(Y \mid X),Q(Y \mid X)): = \sum_{x}Pr\Big( X = x\Big)\big[ KL( P(Y \mid X = x),Q(Y \mid X = x))\big]$ denotes the conditional KL-divergence. 
Moreover it is easy to note that for any $\sigma \in \Sigma_{S_m}$ such that $\sigma(i) = a$, we have $KL(\bnu^1_S(\sigma_{i+1} \mid \sigma(1:i)), \bnu^a_S(\sigma_{i+1} \mid \sigma(1:i))) := 0$, for all $i \in [m]$. We also denote the set of possible top-$i$ rankings of set $S$, by $\Sigma_{S_i}$, for all $i \in [m]$.
Now as derived in \eqref{eq:win_lb0} in the proof of Theorem \ref{thm:lb_plpac_win}, we have 

\[
KL(\bnu^1_S(\sigma_1), \bnu^a_S(\sigma_1)) \le \frac{1}{k}\Big( R - \frac{1}{R}\Big)^2.
\]

To bound the remaining terms of \eqref{eq:lb_witf_kl},  note that for all $i \in [m-1]$
\begin{align*}
KL(\bnu^1_S(\sigma_{i+1} & \mid \sigma(1:i)),  \bnu^a_S(\sigma_{i+1} \mid \sigma(1:i))) \\
& = \sum_{\sigma' \in \Sigma_{S_i}}Pr_{\bnu^1}(\sigma')KL(\bnu^1_S(\sigma_{i+1} \mid \sigma(1:i))=\sigma', \bnu^a_S(\sigma_{i+1} \mid \sigma(1:i))=\sigma')\\
& = 	\sum_{\sigma' \in \Sigma_{S_i}\mid a \notin \sigma'}\Bigg[\prod_{j = 1}^{i}\Bigg(\dfrac{\theta_{\sigma'_j}^1}{\theta_S^1 - \sum_{j'=1}^{j-1}\theta^1_{\sigma'_{j'}}}\Bigg)\Bigg]\dfrac{1}{k-i}\Big( R - \frac{1}{R}\Big)^2 = \dfrac{1}{k}\Big( R - \frac{1}{R}\Big)^2\\	
\end{align*}
where $\theta_S^1 = \sum_{j' \in S}\theta^1_{j'}$.
Thus applying above in \eqref{eq:lb_witf_kl} we get:

\begin{align}
\label{eq:lb_witf_kl2}
\nonumber KL& (\bnu^1_S, \bnu^a_S)\\
\nonumber & = KL(\bnu^1_S(\sigma_1), \bnu^a_S(\sigma_1)) + \cdots + KL(\bnu^1_S(\sigma_m \mid \sigma(1:m-1)), \bnu^a_S(\sigma_m \mid \sigma(1:m-1))) \\
& \le \frac{m}{k}\Big( R - \frac{1}{R}\Big)^2 \le \frac{m}{k}256\epsilon^2 ~~\Bigg[\text{since } \Bigg( R - \frac{1}{R} \Bigg) = \frac{8\epsilon}{(1-4\epsilon^2)} \le 16\epsilon, \forall \epsilon \in [0,\frac{1}{\sqrt{8}}] \Bigg]
\end{align}

Eqn. \eqref{eq:lb_witf_kl2} gives the main result to derive Theorem \ref{thm:lb_pacpl_rnk} as it shows an $m$-factor blow up in the KL-divergence terms owning to top-$m$ ranking feedback. 




Now, consider $\cE_0 \in \cF_\tau$ be an event such that the algorithm $A$ returns the element $i = 1$, and let us analyse the left hand side of \eqref{eq:FI_a} for $\cE = \cE_0$. Clearly, $A$ being an $(\epsilon,\delta)$-PAC  algorithm, we have $Pr_{\bnu^1}(\cE_0) > 1-\delta$, and $Pr_{\bnu^a}(\cE_0) < \delta$, for any suboptimal arm $a \in [n]\setminus\{1\}$. Then we have:

\begin{align}
\label{eq:win_lb2_tr}
kl(Pr_{\bnu^1}(\cE_0),Pr_{\bnu^a}(\cE_0)) \ge kl(1-\delta,\delta) \ge \ln \frac{1}{2.4\delta}
\end{align}

where the last inequality follows due to Equation $(3)$ of \cite{Kaufmann+16_OnComplexity}.

Now applying \eqref{eq:FI_a} and \eqref{eq:win_lb2_tr} for each modified bandit \textbf{Instance-$\bnu^a$}, and summing over all suboptimal items $a \in [n]\setminus \{1\}$ we get,

\begin{align}
\label{eq:win_lb2.5_tr}
\sum_{a = 2}^{n}\sum_{\{S \in \cA \mid a \in S\}}\E_{\bnu^1}[N_S(\tau_A)]KL(\bnu^1_S,\bnu^a_S) \ge (n-1)\ln \frac{1}{2.4\delta}.
\end{align}

Moreover, using \eqref{eq:lb_witf_kl2}, the term in the right hand side of \eqref{eq:win_lb2.5_tr} can be further upper bounded as:

\begin{align}
\label{eq:win_lb3_tr}
\nonumber & \sum_{a = 2}^{n}\sum_{\{S \in \cA \mid a \in S\}} \E_{\bnu^1}[N_S(\tau_A)]KL(\bnu^1_S,\bnu^a_S) \le \sum_{S \in \cA}\E_{\bnu^1}[N_S(\tau_A)]\sum_{\{a \in S \mid a \neq 1\}}\frac{m}{k}256\epsilon^2\\
& = \sum_{S \in \cA}\E_{\bnu^1}[N_S(\tau_A)]{\big( k - \1(1 \in S)\big)}\Big( \frac{m}{k}256\epsilon^2 \Big) \le \sum_{S \in \cA}\E_{\bnu^1}[N_S(\tau_A)]m\Big( 256\epsilon^2 \Big)
\end{align}

Finally noting that $\tau_A = \sum_{S \in \cA}[N_S(\tau_A)]$, combining \eqref{eq:win_lb2.5_tr} and \eqref{eq:win_lb3_tr}, we get 

\begin{align*}
m\Big( 256\epsilon^2 \Big)\E_{\bnu^1}[\tau_A] =  \sum_{S \in \cA}\E_{\bnu^1}[N_S(\tau_A)]m\Big( 256\epsilon^2 \Big) \ge (n-1)\ln \frac{1}{2.4\delta}.
\end{align*}
Thus above construction shows the existence of a problem instance $\bnu = \bnu^1$, such that $\E_{\bnu^1}[\tau_A] = \Omega(\frac{n}{m\epsilon^2}\ln \frac{1}{2.4\delta})$, which concludes the proof.

\end{proof}

\subsection{Proof of Lemma \ref{lem:rb}}
\label{app:rb}

\lemrb*

\begin{proof}
Let us denote $\hat i := \arg \max_{i \in S}q_i$ to be the item (note that it need not be unique) that appears in the top-$m$ set for maximum number of times in $t$ rounds of battle.
Note that, after the battle of any round $\tau \in [t]$, $\sigma_\tau$ chooses exactly $m$ distinct items in the top-$m$ set $S^\tau_m \subseteq S$. Thus $t$ rounds of feedback places exactly $mt$ items in the top-$m$ slots, i.e. $\sum_{i \in S}q_i = mt$.
Now at any round $\tau$, since an item $i \in S$ can appear in $S^\tau_m$ at most once, and $\sum_{i \in S}q_i = mt$, item $\hat i$ must be selected for at least $\frac{mt}{k}$ many rounds in the top-$m$ set implying that $q_{\hat i} \ge \frac{mt}{k}$ (as we have $|S| = k$). 
\end{proof}

\subsection{Proof of Theorem \ref{thm:trace_best_fr}}
\label{app:frub_trc}


\ubtrcfr*

\begin{proof}
We start by analyzing the required sample complexity first. Note that the `while loop' of Algorithm \ref{alg:trace_bb_mod} always discards away $k-1$ items per iteration. Thus, $n$ being the total number of items, the `while loop' can be executed for at most $\lceil \frac{n}{k-1} \rceil $ many number of iterations. Clearly, the sample complexity of each iteration being $t = \frac{2k}{m\epsilon^2}\ln \frac{n}{2\delta}$, the total sample complexity of the algorithm becomes $\big( \lceil \frac{n}{k-1} \rceil \big)\frac{2k}{m\epsilon^2}\ln \frac{n}{2\delta} \le \big( \frac{n}{k-1} + 1 \big)\frac{2k}{m\epsilon^2}\ln \frac{n}{2\delta} = \big( n + \frac{n}{k-1} + k \big)\frac{2}{m\epsilon^2}\ln \frac{n}{2\delta} = O(\frac{n}{m\epsilon^2}\ln \frac{n}{\delta})$.

We now proceed to prove the $(\epsilon,\delta)$-{PAC} correctness of the algorithm. 
As argued before, the `while loop' of Algorithm \ref{alg:trace_bb_mod} can run for maximum $\lceil \frac{n}{k-1} \rceil $ many number of iterations, say $\ell = 1,2, \ldots, \lceil \frac{n}{k-1} \rceil $, and let us denote the corresponding set $\cA$ of iteration $\ell$ as $\cA_\ell$. Same as before, our idea is to retain the estimated best item as the \emph{`running winner'} in $r_\ell$ and compare it with the \emph{`empirical best item'} of $\cA_\ell$ at every $\ell$.
We start by noting the following important property of item $c_\ell$ for any iteration $\ell$:

\begin{lem}
\label{lem:rb_cl}
Suppose $q_{i}: = \sum_{\tau = 1}^{t}\1(i \in \cA^\tau_{\ell m})$ denotes the number of times item $i$ appeared in the top-$m$ ranking in $t$ iterations, and let $B_\ell \subseteq \cA_\ell$ is defined as $B_\ell := \{i \in \cA_\ell \mid q_i = \max_{j \in \cA_\ell}q_j\}$, that denotes the subset of items in $\cA_\ell$ which are selected in the top $m$ ranking for maximum number of times in $t$ rounds of battle on set $\cA_\ell$. Then $c_\ell \in B_\ell$.
\end{lem}

\begin{proof}
We prove the claim by contradiction. Suppose, $c_\ell \notin B_\ell$ and consider any item $\hat i \in B_\ell$. Then by definition, $q_{\hat i} > q_{c_\ell}$. But in that case following our rank breaking update (see Algorithm \ref{alg:updt_win}) implies that $w_{\hat i c_\ell} > w_{c_\ell \hat i}$, since item $\hat i$ is ranked higher than item $c_\ell$ for at least $(q_{\hat i} - q_{c_\ell})>0$ many rounds of battle. Now consider any other item $j \in \cA_\ell$. Note that $j$ can belong to either of these two cases:

\textbf{Case 1. ($j \notin B_\ell$)} Following the same argument as above (i.e. for $\hat i$ vs $c_\ell)$, we again have $w_{\hat i j} > w_{j \hat i}$, whereas for $c_\ell$ vs $j$, either $w_{c_\ell j} > w_{j c_\ell}$, or $w_{c_\ell j} < w_{j c_\ell}$, both cases are plausible. Thus we get: $\1(w_{\hat i j} > w_{j \hat i}) = 1 \ge \1(w_{c_\ell j} > w_{j c_\ell})$.

\textbf{Case 2. ($j \in B_\ell$)} In this case since $j \in B_\ell$, again following the same argument as for $\hat i$ vs $c_\ell$, we here have $w_{j c_\ell} > w_{c_\ell j}$; whereas for $\hat i$ vs $j$, either $w_{\hat i j} > w_{j \hat i}$, or $w_{\hat i j} < w_{j \hat i}$, both cases are plausible. Thus we get: $\1(w_{\hat i j} > w_{j \hat i}) \ge 0 = \1(w_{c_\ell j} > w_{j c_\ell})$.

Combining the results of Case $1$ and $2$ along with $w_{\hat i c_\ell} > w_{c_\ell \hat i}$, we get that $\sum_{j \in \cA\setminus\{\hat i\}}~\1\big(w_{\hat ij} \ge w_{j\hat i}\big) > \sum_{j \in \cA\setminus\{c_\ell\}}~\1\big(w_{c_\ell j} \ge w_{jc_\ell}\big)$. But this violates the fact that $c_\ell$ is defined as $c_\ell := \underset{i \in \cA_\ell}{\text{argmax}} \sum_{j \in \cA\setminus\{i\}}~\1\big(w_{ij} \ge w_{ji}\big)$  which leads to a contradiction. Then our initial assumption has to be wrong and $c_\ell \in B_\ell$, which concludes the proof.
\end{proof}

The next crucial observation lies in noting that, the estimated best item $r$ (`running winner') gets updated as per the following lemma:

\begin{lem}
\label{lem:c_vs_r2}
At any iteration $\ell = 1,2 \ldots, \bigg \lfloor \frac{n}{k-1} \bigg \rfloor$, for any set $\cA_\ell$, nwith probability at least $(1-\frac{\delta}{2n})$, Algorithm \ref{alg:trace_bb} retains $r_{\ell+1} \leftarrow r_\ell$ if $p_{c_\ell r_\ell } \le \frac{1}{2}$, and set $r_{\ell+1} \leftarrow c_\ell$ if $p_{c_\ell r_\ell} \ge \frac{1}{2} + \epsilon$.
\end{lem}

\begin{proof} The main observation lies in proving that at any iteration $\ell$, $w_{c_\ell r_\ell} + w_{r_\ell c_\ell} \ge \frac{mt}{k}$. We argue this as follows:
Firstly note that by Lemma \ref{lem:rb} and \ref{lem:rb_cl}, $c_\ell \in \cB_\ell$ (Lemma \ref{lem:rb_cl}) and hence it must have appeared in top-$m$ positions for at least $\frac{mt}{k}$ times (Lemma \ref{lem:rb}). But the rank breaking update ensures that every element in top-$m$ position gets updated for exactly $k$ times (it loses to all elements preceding it in the top-$m$ ranking and wins over the rest).  
Define $n_{ij} = w_{ij} + w_{ji}$ to be the number of times item $i$ and $j$ are compared after rank-breaking, $i,j \in \cA_\ell$. Clearly $0 \le n_{ij} \le tk$ and $n_{ij} = n_{ji}$. 
Now using above argument we have that $n_{c_\ell r_\ell} = w_{c_\ell r_\ell} + w_{r_\ell c_\ell} \ge \frac{mt}{k}$.
We are now proof the claim using the following two case analyses: 


\textbf{Case 1.} (If $p_{c_\ell r_\ell} \le \frac{1}{2}$, \algtrc\, retains $r_{\ell+1} \leftarrow r_\ell$): Note that \algtrc\, replaces $r_{\ell+1}$ by $c_\ell$ only if $\hp_{c_\ell,r_\ell} > \frac{1}{2} + \frac{\epsilon}{2} $, but this happens with probability: 

\begin{align*}
& Pr\Bigg( \bigg\{ \hp_{c_\ell r_\ell} > \frac{1}{2} + \frac{\epsilon}{2} \bigg\} \Bigg) \\
& = Pr\Bigg( \bigg\{ \hp_{c_\ell r_\ell} > \frac{1}{2} + \frac{\epsilon}{2} \bigg\} \cap \bigg\{ n_{c_\ell r_\ell} \ge \frac{mt}{k} \bigg\}\Bigg) \hspace*{5pt}+\\
& \hspace*{1.5in} \cancelto{0}{Pr\Bigg(\bigg\{ n_{c_\ell r_\ell} < \frac{mt}{k} \bigg\}\Bigg)}Pr\Bigg( \bigg\{ \hp_{c_\ell r_\ell} > \frac{1}{2} + \frac{\epsilon}{2} \bigg\} \Big | \bigg\{ n_{c_\ell r_\ell} < \frac{mt}{k} \bigg\}\Bigg)\\
& \le Pr\Bigg( \bigg\{ \hp_{c_\ell r_\ell} - p_{c_\ell r_\ell} > \frac{\epsilon}{2} \bigg\} \cap \bigg\{ n_{c_\ell r_\ell} \ge \frac{mt}{k} \bigg\} \Bigg) \le \exp\Big( -2\frac{mt}{k}\bigg(\frac{\epsilon}{2}\bigg)^2 \Big) = \frac{\delta}{2n},
\end{align*}
where the first inequality follows as $p_{c_\ell r_\ell} \le \frac{1}{2}$, and the second inequality is simply by applying Lemma \ref{lem:pl_simulator} with $\eta = \frac{\epsilon}{2}$ and $v = \frac{mt}{k}$.
We now proceed to the second case:

\textbf{Case 2.} (If $p_{c_\ell r_\ell} \ge \frac{1}{2} + \epsilon$, \algtrc\,  sets $r_{\ell+1} \leftarrow c_\ell$): Again recall that \algtrc\, retains $r_{\ell+1}  \leftarrow r_\ell$ only if $\hp_{c_\ell,r_\ell} \le \frac{1}{2} + \frac{\epsilon}{2} $. In this case, that happens with probability:

\begin{align*}
& Pr\Bigg( \bigg\{ \hp_{c_\ell r_\ell} \le \frac{1}{2} + \frac{\epsilon}{2} \bigg\} \Bigg) \\
& = Pr\Bigg( \bigg\{ \hp_{c_\ell r_\ell} \le \frac{1}{2} + \frac{\epsilon}{2} \bigg\} \cap \bigg\{ n_{c_\ell r_\ell} \ge \frac{mt}{k} \bigg\}\Bigg) \hspace*{5pt}+\\
& \hspace*{1.5in} \cancelto{0}{Pr\Bigg(\bigg\{ n_{c_\ell r_\ell} < \frac{mt}{k} \bigg\}\Bigg)}Pr\Bigg( \bigg\{ \hp_{c_\ell r_\ell} \le \frac{1}{2} + \frac{\epsilon}{2} \bigg\} \Big | \bigg\{ n_{c_\ell r_\ell} < \frac{mt}{k} \bigg\}\Bigg)\\
& = Pr\Bigg( \bigg\{ \hp_{c_\ell r_\ell} \le \frac{1}{2} + \epsilon - \frac{\epsilon}{2} \bigg\} \cap \bigg\{ n_{c_\ell r_\ell} \ge \frac{mt}{k} \bigg\} \Bigg) \\
& \le Pr\Bigg( \bigg\{ \hp_{c_\ell r_\ell} - p_{c_\ell r_\ell} \le - \frac{\epsilon}{2} \bigg\} \cap \bigg\{ n_{c_\ell r_\ell} \ge \frac{mt}{k} \bigg\} \Bigg) \le \exp\Big( -2\frac{mt}{k}\bigg(\frac{\epsilon}{2}\bigg)^2 \Big) = \frac{\delta}{2n},
\end{align*}
where the first inequality holds as $p_{c_\ell r_\ell} \ge \frac{1}{2} + \epsilon$, and the second one is simply by applying Lemma \ref{lem:pl_simulator} with $\eta = \frac{\epsilon}{2}$ and $v = \frac{mt}{k}$. Combining the above two cases concludes the proof.

\end{proof}

The rest of the proof follows exactly same as that of Theorem \ref{thm:trace_best}. We include the details for completeness.
Given Algorithm \ref{alg:trace_bb_mod} satisfies Lemma \ref{lem:c_vs_r2}, and taking union bound over $(k-1)$ elements in $\cA_\ell \setminus\{r_\ell\}$, we get that with probability at least $\bigg(1-\frac{(k-1)\delta}{2n}\bigg)$,

\begin{align}
\label{eq:r_vs_c2}
p_{r_{\ell+1}r_\ell} \ge \frac{1}{2} \text{ and, } p_{r_{\ell+1}c_\ell} \ge \frac{1}{2} - \epsilon.
\end{align} 

Above clearly suggests that for each iteration $\ell$, the estimated `best' item $r_\ell$ only gets improved as $p_{r_{\ell+1}r_\ell} \ge 0$. Let, $\ell_*$ denotes the specific iteration such that $1 \in \cA_\ell$ for the first time, i.e. $\ell_* = \min\{ \ell \mid 1 \in \cA_\ell \}$. Clearly $\ell_* \le \lceil \frac{n}{k-1} \rceil$. 

Now \eqref{eq:r_vs_c2} suggests that with probability at least $(1-\frac{(k-1)\delta}{2n})$, $p_{r_{\ell_*+1}1} \ge \frac{1}{2} - \epsilon$. Moreover \eqref{eq:r_vs_c2} also suggests that for all $\ell > \ell_*$, with probability at least $(1-\frac{(k-1)\delta}{2n})$, $p_{r_{\ell+1}r_\ell} \ge \frac{1}{2}$, which implies for all $\ell > \ell_*$, $p_{r_{\ell+1}1} \ge \frac{1}{2} - \epsilon$ as well. 

Note that above holds due to the following transitivity property of the Plackett-Luce model: For any three items $i_1,i_2,i_3 \in [n]$, if $p_{i_1i_2}\ge \frac{1}{2}$ and $p_{i_2i_3}\ge \frac{1}{2}$, then we have $p_{i_1i_3}\ge \frac{1}{2}$ as well. This argument finally leads to $p_{r_*1} \ge \frac{1}{2} - \epsilon$. Since failure probability at each iteration $\ell$ is at most $\frac{(k-1)\delta}{2n}$, and Algorithm \ref{alg:trace_bb_mod} runs for maximum $\lceil \frac{n}{k-1} \rceil$ many number of iterations, using union bound over $\ell$, the total failure probability of the algorithm is at most $\lceil \frac{n}{k-1} \rceil \frac{(k-1)\delta}{2n} \le (\frac{n}{k-1}+1)\frac{(k-1)\delta}{2n} = \delta\Big( \frac{n+k-1}{2n} \Big) \le \delta$ (since $k \le n$). This concludes the correctness of the algorithm showing that it indeed returns an $\epsilon$-best element $r_*$ such that $p_{r_*1} \ge \frac{1}{2} - \epsilon$ with probability at least $1-\delta$. 
\end{proof}

\subsection{Proof of Theorem \ref{thm:div_bb_fr}}
\label{app:frub_div}

\ubdivfr*

\begin{proof}
For the notational convenience we will use $\tp_{ij} = p_{ij} - \frac{1}{2}, \, \forall i,j \in [n]$. 

We first analyze the required sample complexity of the algorithm. 
For clarity of notation, we will denote the set $S$ at iteration $\ell$ by $S_\ell$.
Note that at any iteration $\ell$, any set $\cG_g$ is played for exactly $t= \frac{4k}{m\epsilon_\ell^2}\ln \frac{2k}{\delta_\ell}$ many number of times. Also since the algorithm discards away exactly $k-1$ items from each set $\cG_g$, hence the maximum number of iterations possible is $\lceil  \ln_k n \rceil$. Now at any iteration $\ell$, since $G = \Big\lfloor \frac{|S_\ell|}{k} \Big\rfloor < \frac{|S_\ell|}{k}$, the total sample complexity for iteration $\ell$ is at most $\frac{|S_\ell|}{k}t \le \frac{4n}{mk^{\ell-1}\epsilon_\ell^2}\ln \frac{2k}{\delta_\ell}$, as $|S_\ell| \le \frac{n}{k^\ell}$ for all $\ell \in [\lfloor  \ln_k n \rfloor]$. Also note that for all but last iteration $\ell \in [\lfloor  \ln_k n \rfloor]$, $\epsilon_\ell = \frac{\epsilon}{8}\bigg( \frac{3}{4} \bigg)^{\ell-1}$, and $\delta_\ell = \frac{\delta}{2^{\ell+1}}$.
Moreover for the last iteration $\ell = \lceil  \ln_k n \rceil$, the sample complexity is clearly $t= \frac{4k}{m(\epsilon/2)^2}\ln \frac{4k}{\delta}$, as in this case $\epsilon_\ell = \frac{\epsilon}{2}$, and $\delta_\ell = \frac{\delta}{2}$, and $|S| = k$.
Thus the total sample complexity of Algorithm \ref{alg:div_bb_mod} is given by 

\begin{align*}
\sum_{\ell = 1}^{\lceil  \ln_k n \rceil}\frac{|S_\ell|}{m(\epsilon_\ell/2)^2}& \ln \frac{2k}{\delta_\ell}  
 \le \sum_{\ell = 1}^{\infty}\frac{4n}{mk^\ell\bigg(\frac{\epsilon}{8}\big(\frac{3}{4}\big)^{\ell-1}\bigg)^2}k\ln \frac{k 2^{\ell+1}}{\delta} + \frac{16k}{m\epsilon^2}\ln \frac{4k}{\delta}\\ 
& \le \frac{256n}{m\epsilon^2}\sum_{\ell = 1}^{\infty}\frac{16^{\ell-1}}{(9k)^{\ell-1}}\Big( \ln \frac{k}{\delta} + {(\ell+1)} \Big) + \frac{16k}{m\epsilon^2}\ln \frac{4k}{\delta} \\
& \le \frac{256n}{m\epsilon^2}\ln \frac{k}{\delta}\sum_{\ell = 1}^{\infty}\frac{4^{\ell-1}}{(9k)^{\ell-1}}\Big( {3\ell} \Big) + \frac{16k}{m\epsilon^2}\ln \frac{4k}{\delta} = O\bigg(\frac{n}{m\epsilon^2}\ln \frac{k}{\delta}\bigg) ~[\text{for any } k > 1].\end{align*}

Above proves the sample complexity bound of Theorem \ref{thm:div_bb_fr}. We now proceed to prove the $(\epsilon,\delta)$-{PAC} correctness of the algorithm. We start by making the following observations:

\begin{lem}
\label{lem:divbat_n1} 
Consider any particular set $\cG_g$ at any iteration $\ell \in \lfloor \frac{n}{k} \rfloor$ and define $q_{i}: = \sum_{\tau = 1}^{t}\1(i \in \cG^\tau_{g m})$ as the number of times any item $i \in \cG_g$ appears in the top-$m$ rankings when items in the set $\cG_g$ is made to battle for $t$ rounds. Then if $i_g := \arg \max_{i \in \cG_g}\theta_i$, then with probability at least $\Big(1-\frac{\delta_\ell}{2k}\Big)$, one can show that $q_{i_g} > (1-\eta)\frac{mt}{k}$, for any $\eta \in \big(\frac{3}{32\sqrt 2},1 \big]$.
\end{lem} 

\begin{proof}
Fix any iteration $\ell$ and a set $\cG_g$, $g \in 1,2, \ldots, G$. Define $i^\tau: = \1(i \in \cG_{g m}^\tau)$ as the indicator variable if $i^{th}$ element appeared in the top-$m$ ranking at iteration $\tau \in [t]$.  Recall the definition of {TR} feedback model (Sec. \ref{sec:feed_mod}). Using this we get $\E[i_g^\tau] = Pr(\{i_g \in \cG_{g m}^\tau\}) = Pr\big( \exists j \in [m] ~|~ \sigma(j) = i_g \big) = \sum_{j = 1}^{m}Pr\big( \sigma(j) = i_g \Big) = \sum_{j = 0}^{m-1}\frac{1}{k-j} \ge \frac{m}{k}$, as $Pr(\{i_g | S\}) = \frac{\theta_{i_g}}{\sum_{j \in S}\theta_j} \ge \frac{1}{|S|}$ for any $S \subseteq [\cG_g]$, as $i_g := \arg \max_{i \in \cG_g}\theta_i$ is the best item of set $\cG_g$. Hence $\E[q_{i_g}] = \sum_{\tau = 1}^{t}\E[i_g^\tau] \ge \frac{mt}{k}$. 


Now applying Chernoff-Hoeffdings bound for $w_{i_g}$, we get that for any $\eta \in (\frac{3}{32},1]$, 

\begin{align*}
Pr\Big( q_{i_g} \le (1-\eta)\E[q_{i_g}] \Big) & \le \exp(- \frac{\E[q_{i_g}]\eta^2}{2}) \le \exp(- \frac{mt\eta^2}{2k}) \\
& = \exp\bigg(- \frac{2\eta^2}{\epsilon_\ell^2} \ln \bigg( \frac{2k}{\delta_\ell} \bigg) \bigg) = \exp\bigg(- \frac{(\sqrt 2\eta)^2}{\epsilon_\ell^2} \ln \bigg( \frac{2k}{\delta_\ell} \bigg) \bigg) \\
& \le \exp\bigg(- \ln \bigg( \frac{2k}{\delta_\ell} \bigg) \bigg) \le \frac{\delta_\ell}{2k} ,
\end{align*}
where the second last inequality holds as $\eta \ge \frac{3}{32\sqrt 2}$ and $\epsilon_\ell \le \frac{3}{32}$, for any iteration $\ell \in \lceil \ln n \rceil$; in other words for any $\eta \ge \frac{3}{32\sqrt 2}$, we have $\frac{\sqrt 2\eta}{\epsilon_\ell} \ge 1$ which leads to the second last inequality. Thus we finally derive that 
with probability at least $\Big(1-\frac{\delta_\ell}{2k}\Big)$, one can show that $q_{i_g} > (1-\eta)\E[q_{i_g}] \ge (1-\eta)\frac{tm}{k}$, and
the proof follows henceforth.
\end{proof}

In particular, fixing $\eta = \frac{1}{2}$ in Lemma \ref{lem:hv_n1}, we get that with probability at least $\big(1-\frac{\delta_\ell}{2}\big)$,  $q_{i_g} > (1-\frac{1}{2})\E[w_{i_g}] > \frac{mt}{2k}$. 
Note that, for any round $\tau \in [t]$, whenever an item $i \in \cG_g$ appears in the top-$m$ set $\cG_{gm}^\tau$, then the rank breaking update ensures that every element in the top-$m$ set gets compared with rest of the $k-1$ elements of $\cG_g$. Based on this observation, we now prove that for any set $\cG_g$, its best item $i_g$ is retained as the winner $c_g$ with probability at least $\big( 1 - \frac{\delta_\ell}{2}\big)$. More formally, first thing to observe is:

\begin{lem}
\label{lem:divbat_n2} 
Consider any particular set $\cG_g$ at any iteration $\ell \in \lfloor \frac{n}{k} \rfloor$. If $i_g \leftarrow \arg \max_{i \in \cG_g}\theta_i$, then with probability at least $\Big(1-{\delta_\ell}\Big)$, $\hp_{i_g j} + \frac{\epsilon_\ell}{2} \ge \frac{1}{2}$ for all $\epsilon_\ell$-optimal item $\forall j \in \cG_g$ such that $p_{i_g j} \in \big( \frac{1}{2}, \frac{1}{2} + \epsilon_\ell \big]$, and $\hp_{i_g j} - \frac{\epsilon_\ell}{2} \ge \frac{1}{2}$ for all non $\epsilon_\ell$-optimal item $ j \in \cG_g$ such that $p_{i_g j} > \frac{1}{2} + \epsilon_\ell$.
\end{lem}

\begin{proof}
With top-$m$ ranking feedback, the crucial observation lies in the fact that at any round $\tau \in [t]$, whenever an item $i \in \cG_g$ appears in the top-$m$ set $\cG_{gm}^\tau$, then the rank breaking update ensures that every element in the top-$m$ set gets compared with each of the rest of the $k-1$ elements of $\cG_g$ - it defeats to every element preceding item in $\sigma \in \Sigma_{\cG_{gm}}$, and wins over the rest.  
Therefore defining $n_{ij} = w_{ij} + w_{ji}$ to be the number of times item $i$ and $j$ are compared after rank-breaking, $i,j \in \cG_g$. Clearly $n_{ij} = n_{ji}$, and $0 \le n_{ij} \le tk$. Moreover, from Lemma \ref{lem:divbat_n1} with $\eta = \frac{1}{2}$, we have that $n_{i_g j} \ge \frac{mt}{2k}$.
Given the above arguments in place let us analyze the probability of a `bad event' that indedd:

\textbf{Case 1.} $j$ is $\epsilon_\ell$-optimal with respect to $i_g$, i.e. $p_{i_g j} \in \big(\frac{1}{2}, \frac{1}{2} + \epsilon_\ell \big]$. Then we have 

\begin{align*}
Pr\Bigg( \bigg\{ \hp_{i_g j} + \frac{\epsilon_\ell}{2} < \frac{1}{2}  \bigg\} & \cap \bigg\{ n_{i_g j} \ge \frac{mt}{2k} \bigg\} \Bigg)
 = Pr\Bigg( \bigg\{ \hp_{i_g j} < \frac{1}{2} - \frac{\epsilon_\ell}{2} \bigg\} \cap \bigg\{ n_{i_g j} = \frac{mt}{2k} \bigg\} \Bigg)\\
& \le Pr\bigg( \bigg\{ \hp_{i_g j} - p_{i_g j} <  - \frac{\epsilon_\ell}{2}\bigg\} \cap \bigg\{ n_{i_g j} = \frac{mt}{2k} \bigg\} \bigg)\\
&  \le \exp\Big( -2\frac{mt}{2k}{(\epsilon_\ell/2)}^2 \Big) \bigg) = \frac{\delta_\ell}{2k},
\end{align*}

where the first inequality follows as $p_{i_g j} > \frac{1}{2}$, and the second inequality follows from Lemma \ref{lem:pl_simulator} with $\eta = \frac{\epsilon_\ell}{2}$ and $v = \frac{mt}{2k}$.

\textbf{Case 2.} $j$ is non $\epsilon_\ell$-optimal with respect to $i_g$, i.e. $p_{i_g j} > \frac{1}{2} + \epsilon_\ell$. Similar to before, we have

\begin{align*}
Pr\Bigg( \bigg\{ \hp_{i_g j} - \frac{\epsilon_\ell}{2} < \frac{1}{2}  \bigg\} & \cap \bigg\{ n_{i_g j} \ge \frac{mt}{2k} \bigg\} \Bigg)
 = Pr\Bigg( \bigg\{ \hp_{i_g j} < \frac{1}{2} + \frac{\epsilon_\ell}{2} \bigg\} \cap \bigg\{ n_{i_g j} = \frac{mt}{2k} \bigg\} \Bigg)\\
& \le Pr\bigg( \bigg\{ \hp_{i_g j} - p_{i_g j} <  - \frac{\epsilon_\ell}{2}\bigg\} \cap \bigg\{ n_{i_g j} = \frac{mt}{2k} \bigg\} \bigg)\\
&  \le \exp\Big( -2\frac{mt}{2k}{(\epsilon_\ell/2)}^2 \Big) \bigg) = \frac{\delta_\ell}{2k},
\end{align*}

where the third last inequality follows since in this case $p_{i_g j} > \frac{1}{2} + \epsilon_\ell$, and the last inequality follows from Lemma \ref{lem:pl_simulator} with $\eta = \frac{\epsilon_\ell}{2}$ and $v = \frac{mt}{2k}$.

Let us define the event $\cE: = \bigg\{ \exists j \in \cG_g  \text{ such that } \hp_{i_g j} + \frac{\epsilon_\ell}{2} < \frac{1}{2}, p_{i_g j} \in \big( \frac{1}{2}, \frac{1}{2} + \epsilon_\ell \big] \text{ or } \hp_{i_g j} -  \frac{\epsilon_\ell}{2} < \frac{1}{2}, p_{i_g j} > \frac{1}{2} + \epsilon_\ell  \bigg\}$. Then by combining Case $1$ and $2$, we get
\begin{align*}
 Pr\Big( \cE \Big) & = Pr\Bigg( \cE \cap \bigg\{ n_{i_g j} \ge \frac{mt}{2k} \bigg\} \Bigg) + Pr\Bigg( \cE \cap \bigg\{ n_{i_g j} < \frac{mt}{2k} \bigg\} \Bigg)\\
& \le \sum_{j \in \cG_g \text{ s.t. } p_{i_g j} \in \big(\frac{1}{2}, \frac{1}{2} + \epsilon_\ell \big]}Pr\Bigg( \bigg\{ \hp_{i_g j} + \frac{\epsilon_\ell}{2} < \frac{1}{2} \bigg\} \cap \bigg\{ n_{i_g j} \ge \frac{mt}{2k} \bigg\} \Bigg) \\
& + \sum_{j \in \cG_g \text{ s.t. } p_{i_g j} > \frac{1}{2} + \epsilon_\ell}Pr\Bigg( \bigg\{ \hp_{i_g j} - \frac{\epsilon_\ell}{2} < \frac{1}{2} \bigg\} \cap \bigg\{ n_{i_g j} \ge \frac{mt}{2k} \bigg\} \Bigg) + Pr\Bigg( \bigg\{ n_{i_g j} < \frac{mt}{2k} \bigg\} \Bigg)\\
& \le \frac{(k-1)\delta_\ell}{2k} + \frac{\delta_\ell}{2k} \le {\delta_\ell}
\end{align*}

where the last inequality follows from the above two case analyses and Lemma \ref{lem:divbat_n1}.

\end{proof}



Given Lemma \ref{lem:divbat_n2} in place, let us now analyze with what probability the algorithm can select a non $\epsilon_\ell$-optimal item $j \in \cG_g$ as $c_g$ at any iteration $\ell \in \lceil \frac{n}{k} \rceil$. For any set $\cG_g$ (or set $S$ for the last iteration $\ell = \lceil \frac{n}{k} \rceil$), we define the set of non $\epsilon_\ell$-optimal element $\cO_g = \{ j \in \cG_g \mid p_{i_g j} > \frac{1}{2} + \epsilon_\ell \}$, and recall the event $\cE: = \bigg\{ \exists j \in \cG_g  \text{ such that } \hp_{i_g j} + \frac{\epsilon_\ell}{2} < \frac{1}{2}, p_{i_g j} \in \big( \frac{1}{2}, \frac{1}{2} + \epsilon_\ell \big] \text{ or } \hp_{i_g j} -  \frac{\epsilon_\ell}{2} < \frac{1}{2}, p_{i_g j} > \frac{1}{2} + \epsilon_\ell  \bigg\}$. Then we have 

\begin{align}
\label{eq:divbat_1}
\nonumber Pr(c_g \in \cO_g) & \le Pr\Bigg( \bigg\{ \exists j \in \cG_g, \hp_{i_g j} + \frac{\epsilon_\ell}{2} < \frac{1}{2} \bigg\} \cup \bigg\{ \exists j \in \cO_g, \hp_{j i_g} + \frac{\epsilon_\ell}{2} \ge \frac{1}{2} \bigg\} \Bigg) \\
\nonumber & \le Pr\Bigg( \cE \cup \bigg\{ \exists j \in \cO_g, \hp_{j i_g} + \frac{\epsilon_\ell}{2} \ge \frac{1}{2} \bigg\} \Bigg) \\
\nonumber & = Pr \Big( \cE \Big) + Pr\Bigg( \bigg\{ \exists j \in \cO_g, \hp_{j i_g} + \frac{\epsilon_\ell}{2} \ge \frac{1}{2} \bigg\} \cap \cE^{c} \Bigg)\\ 
& = Pr \Big( \cE \Big) + Pr\Bigg( \bigg\{ \exists j \in \cO_g, \hp_{j i_g} + \frac{\epsilon_\ell}{2} \ge \frac{1}{2} \bigg\} \cap \cE^{c} \Bigg) \le \delta_\ell + 0 = \delta_\ell,
\end{align}

where the last inequality follows from Lemma \ref{lem:divbat_n2}, and the fact that $\hp_{i_g j} - \frac{\epsilon_\ell}{2} \ge \frac{1}{2} \implies \hp_{j i_g} + \frac{\epsilon_\ell}{2} < \frac{1}{2}$. The proof now follows combining all the above parts together.

More formally, for each iteration $\ell$, let us define $g_\ell \in [G]$ to be the set that contains \textit{best item} of the entire set $S$, i.e. $\arg\max_{i \in S}\theta_i \in \cG_{g_\ell}$. Then from \eqref{eq:divbat_1}, with probability at least $(1-\delta_\ell)$,\, $\tp_{c_{g_\ell}i_{g_\ell}} > -\epsilon_\ell$. Now for each iteration $\ell$, recursively applying \eqref{eq:divbat_1} and Lemma \ref{lem:pl_sst} to $\cG_{g_\ell}$, we get that $\tp_{r_*1} > -\Big( \frac{\epsilon}{8} + \frac{\epsilon}{8}\Big(\frac{3}{4}\Big) + \cdots + \frac{\epsilon}{8}\big(\frac{3}{4}\big)^{\lfloor \frac{n}{k} \rfloor} \Big) + \frac{\epsilon}{2} \ge -\frac{\epsilon}{8}\Big( \sum_{i = 0}^{\infty}\big( \frac{3}{4} \big)^{i} \Big) + \frac{\epsilon}{2} = \epsilon$.
(Note that, for above analysis to go through, it is in fact sufficient to consider only the set of iterations $\{\ell \ge \ell_0 \mid \ell_0 = \min\{l \mid 1 \notin \cR_l, \, l \ge 1\} \}$ because prior considering item $1$, it does not matter even if the algorithm mistakes in any of the iteration $\ell < \ell_0$). Thus assuming the algorithm does not fail in any of the iteration $\ell$, we have that $p_{r_*1} > \frac{1}{2} - \epsilon$.

Finally, since at each iteration $\ell$, the algorithm fails with probability at most $\delta_\ell$, the total failure probability of the algorithm is at most $\Big( \frac{\delta}{4} + \frac{\delta}{8} + \cdots + \frac{\delta}{2^{\lceil \frac{n}{k}  \rceil}} \Big) + \frac{\delta}{2} \le \delta$.
This concludes the correctness of the algorithm showing that it indeed returns an $\epsilon$-best element $r_*$ such that $p_{r_*1} \ge \frac{1}{2} - \epsilon$ with probability at least $1-\delta$. 
\end{proof}


%% file: main-arxiv-bbwi.bbl
\begin{thebibliography}{37}
\providecommand{\natexlab}[1]{#1}
\providecommand{\url}[1]{\texttt{#1}}
\expandafter\ifx\csname urlstyle\endcsname\relax
  \providecommand{\doi}[1]{doi: #1}\else
  \providecommand{\doi}{doi: \begingroup \urlstyle{rm}\Url}\fi

\bibitem[Agrawal et~al.(2016)Agrawal, Avandhanula, Goyal, and
  Zeevi]{Agrawal+16}
Shipra Agrawal, Vashist Avandhanula, Vineet Goyal, and Assaf Zeevi.
\newblock A near-optimal exploration-exploitation approach for assortment
  selection.
\newblock 2016.

\bibitem[Ailon et~al.(2014)Ailon, Karnin, and Joachims]{Ailon+14}
Nir Ailon, Zohar~Shay Karnin, and Thorsten Joachims.
\newblock Reducing dueling bandits to cardinal bandits.
\newblock In \emph{ICML}, volume~32, pages 856--864, 2014.

\bibitem[Audibert and Bubeck(2010)]{Audibert+10}
Jean-Yves Audibert and S{\'e}bastien Bubeck.
\newblock Best arm identification in multi-armed bandits.
\newblock In \emph{COLT-23th Conference on Learning Theory-2010}, pages 13--p,
  2010.

\bibitem[Auer et~al.(2002)Auer, Cesa-Bianchi, and Fischer]{Auer+02}
Peter Auer, Nicolo Cesa-Bianchi, and Paul Fischer.
\newblock Finite-time analysis of the multiarmed bandit problem.
\newblock \emph{Machine learning}, 47\penalty0 (2-3):\penalty0 235--256, 2002.

\bibitem[Azari et~al.(2012)Azari, Parkes, and Xia]{Az+12}
Hossein Azari, David Parkes, and Lirong Xia.
\newblock Random utility theory for social choice.
\newblock In \emph{Advances in Neural Information Processing Systems}, pages
  126--134, 2012.

\bibitem[Benson et~al.(2016)Benson, Kumar, and Tomkins]{IIA-relevance16}
Austin~R Benson, Ravi Kumar, and Andrew Tomkins.
\newblock On the relevance of irrelevant alternatives.
\newblock In \emph{Proceedings of the 25th International Conference on World
  Wide Web}, pages 963--973. International World Wide Web Conferences Steering
  Committee, 2016.

\bibitem[Boucheron et~al.(2013)Boucheron, Lugosi, and Massart]{CI_book}
St{\'e}phane Boucheron, G{\'a}bor Lugosi, and Pascal Massart.
\newblock \emph{Concentration inequalities: A nonasymptotic theory of
  independence}.
\newblock Oxford university press, 2013.

\bibitem[Busa-Fekete et~al.(2013)Busa-Fekete, Szorenyi, Cheng, Weng, and
  H{\"u}llermeier]{Busa_top}
R{\'o}bert Busa-Fekete, Balazs Szorenyi, Weiwei Cheng, Paul Weng, and Eyke
  H{\"u}llermeier.
\newblock Top-k selection based on adaptive sampling of noisy preferences.
\newblock In \emph{International Conference on Machine Learning}, pages
  1094--1102, 2013.

\bibitem[Busa-Fekete et~al.(2014{\natexlab{a}})Busa-Fekete, H{\"u}llermeier,
  and Sz{\"o}r{\'e}nyi]{Busa_mallows}
R{\'o}bert Busa-Fekete, Eyke H{\"u}llermeier, and Bal{\'a}zs Sz{\"o}r{\'e}nyi.
\newblock Preference-based rank elicitation using statistical models: The case
  of mallows.
\newblock In \emph{Proceedings of The 31st International Conference on Machine
  Learning}, volume~32, 2014{\natexlab{a}}.

\bibitem[Busa-Fekete et~al.(2014{\natexlab{b}})Busa-Fekete, Sz{\"o}r{\'e}nyi,
  and H{\"u}llermeier]{Busa_aaai}
R{\'o}bert Busa-Fekete, Bal{\'a}zs Sz{\"o}r{\'e}nyi, and Eyke H{\"u}llermeier.
\newblock Pac rank elicitation through adaptive sampling of stochastic pairwise
  preferences.
\newblock In \emph{AAAI}, pages 1701--1707, 2014{\natexlab{b}}.

\bibitem[Chen et~al.(2017)Chen, Gopi, Mao, and Schneider]{ChenSoda+17}
Xi~Chen, Sivakanth Gopi, Jieming Mao, and Jon Schneider.
\newblock Competitive analysis of the top-k ranking problem.
\newblock In \emph{Proceedings of the Twenty-Eighth Annual ACM-SIAM Symposium
  on Discrete Algorithms}, pages 1245--1264. SIAM, 2017.

\bibitem[Chen et~al.(2018)Chen, Li, and Mao]{ChenSoda+18}
Xi~Chen, Yuanzhi Li, and Jieming Mao.
\newblock A nearly instance optimal algorithm for top-k ranking under the
  multinomial logit model.
\newblock In \emph{Proceedings of the Twenty-Ninth Annual ACM-SIAM Symposium on
  Discrete Algorithms}, pages 2504--2522. SIAM, 2018.

\bibitem[Chen and Suh(2015)]{SueIcml+15}
Yuxin Chen and Changho Suh.
\newblock Spectral mle: Top-k rank aggregation from pairwise comparisons.
\newblock In \emph{International Conference on Machine Learning}, pages
  371--380, 2015.

\bibitem[Even-Dar et~al.(2006)Even-Dar, Mannor, and Mansour]{Even+06}
Eyal Even-Dar, Shie Mannor, and Yishay Mansour.
\newblock Action elimination and stopping conditions for the multi-armed bandit
  and reinforcement learning problems.
\newblock \emph{Journal of machine learning research}, 7\penalty0
  (Jun):\penalty0 1079--1105, 2006.

\bibitem[Falahatgar et~al.(2017)Falahatgar, Hao, Orlitsky, Pichapati, and
  Ravindrakumar]{falahatgar_nips}
Moein Falahatgar, Yi~Hao, Alon Orlitsky, Venkatadheeraj Pichapati, and Vaishakh
  Ravindrakumar.
\newblock Maxing and ranking with few assumptions.
\newblock In \emph{Advances in Neural Information Processing Systems}, pages
  7063--7073, 2017.

\bibitem[Gajane et~al.(2015)Gajane, Urvoy, and Cl{\'e}rot]{Adv_DB}
Pratik Gajane, Tanguy Urvoy, and Fabrice Cl{\'e}rot.
\newblock A relative exponential weighing algorithm for adversarial
  utility-based dueling bandits.
\newblock In \emph{Proceedings of the 32nd International Conference on Machine
  Learning}, pages 218--227, 2015.

\bibitem[Jamieson et~al.(2014)Jamieson, Malloy, Nowak, and
  Bubeck]{pmlr-v35-jamieson14}
Kevin Jamieson, Matthew Malloy, Robert Nowak, and Sébastien Bubeck.
\newblock lil' ucb : An optimal exploration algorithm for multi-armed bandits.
\newblock In Maria~Florina Balcan, Vitaly Feldman, and Csaba Szepesvári,
  editors, \emph{Proceedings of The 27th Conference on Learning Theory},
  volume~35 of \emph{Proceedings of Machine Learning Research}, pages 423--439.
  PMLR, 2014.

\bibitem[Jang et~al.(2017)Jang, Kim, Suh, and Oh]{SueIcml+17}
Minje Jang, Sunghyun Kim, Changho Suh, and Sewoong Oh.
\newblock Optimal sample complexity of m-wise data for top-k ranking.
\newblock In \emph{Advances in Neural Information Processing Systems}, pages
  1685--1695, 2017.

\bibitem[Kalyanakrishnan et~al.(2012)Kalyanakrishnan, Tewari, Auer, and
  Stone]{Kalyanakrishnan+12}
Shivaram Kalyanakrishnan, Ambuj Tewari, Peter Auer, and Peter Stone.
\newblock Pac subset selection in stochastic multi-armed bandits.
\newblock In \emph{ICML}, volume~12, pages 655--662, 2012.

\bibitem[Karnin et~al.(2013)Karnin, Koren, and Somekh]{Karnin+13}
Zohar Karnin, Tomer Koren, and Oren Somekh.
\newblock Almost optimal exploration in multi-armed bandits.
\newblock In \emph{International Conference on Machine Learning}, pages
  1238--1246, 2013.

\bibitem[Kaufmann et~al.(2016)Kaufmann, Capp{\'e}, and
  Garivier]{Kaufmann+16_OnComplexity}
Emilie Kaufmann, Olivier Capp{\'e}, and Aur{\'e}lien Garivier.
\newblock On the complexity of best-arm identification in multi-armed bandit
  models.
\newblock \emph{The Journal of Machine Learning Research}, 17\penalty0
  (1):\penalty0 1--42, 2016.

\bibitem[Khetan and Oh(2016)]{KhetanOh16}
Ashish Khetan and Sewoong Oh.
\newblock Data-driven rank breaking for efficient rank aggregation.
\newblock \emph{Journal of Machine Learning Research}, 17\penalty0
  (193):\penalty0 1--54, 2016.

\bibitem[Marden(1996)]{Marden_book}
John~I. Marden.
\newblock \emph{Analyzing and Modeling Rank Data}.
\newblock Chapman and Hall/CRC, 1996.

\bibitem[Mohajer et~al.(2017)Mohajer, Suh, and Elmahdy]{MohajerIcml+17}
Soheil Mohajer, Changho Suh, and Adel Elmahdy.
\newblock Active learning for top-$ k $ rank aggregation from noisy
  comparisons.
\newblock In \emph{International Conference on Machine Learning}, pages
  2488--2497, 2017.

\bibitem[Popescu et~al.(2016)Popescu, Dragomir, Slu{{s}}anschi, and
  St{{a}}n{{a}}{{s}}il{{a}}]{klub16}
Pantelimon~G Popescu, Silvestru Dragomir, Emil~I Slu{{s}}anschi, and Octavian~N
  St{{a}}n{{a}}{{s}}il{{a}}.
\newblock Bounds for {K}ullback-{L}eibler divergence.
\newblock \emph{Electronic Journal of Differential Equations}, 2016, 2016.

\bibitem[Ramamohan et~al.(2016)Ramamohan, Rajkumar, and Agarwal]{Ramamohan+16}
Siddartha~Y Ramamohan, Arun Rajkumar, and Shivani Agarwal.
\newblock Dueling bandits: Beyond condorcet winners to general tournament
  solutions.
\newblock In \emph{Advances in Neural Information Processing Systems}, pages
  1253--1261, 2016.

\bibitem[Saha and Gopalan(2018)]{SG18}
Aadirupa Saha and Aditya Gopalan.
\newblock Battle of bandits.
\newblock In \emph{Uncertainty in Artificial Intelligence}, 2018.

\bibitem[Simchowitz et~al.(2017)Simchowitz, Jamieson, and
  Recht]{SimJamRec17:moderate}
Max Simchowitz, Kevin Jamieson, and Benjamin Recht.
\newblock The simulator: Understanding adaptive sampling in the
  moderate-confidence regime.
\newblock In \emph{Proceedings of The 30th Conference on Learning Theory},
  2017.

\bibitem[Soufiani et~al.(2014)Soufiani, Parkes, and Xia]{AzariRB+14}
Hossein~Azari Soufiani, David~C Parkes, and Lirong Xia.
\newblock Computing parametric ranking models via rank-breaking.
\newblock In \emph{ICML}, pages 360--368, 2014.

\bibitem[Sz{\"o}r{\'e}nyi et~al.(2015)Sz{\"o}r{\'e}nyi, Busa-Fekete, Paul, and
  H{\"u}llermeier]{Busa_pl}
Bal{\'a}zs Sz{\"o}r{\'e}nyi, R{\'o}bert Busa-Fekete, Adil Paul, and Eyke
  H{\"u}llermeier.
\newblock Online rank elicitation for plackett-luce: A dueling bandits
  approach.
\newblock In \emph{Advances in Neural Information Processing Systems}, pages
  604--612, 2015.

\bibitem[Urvoy et~al.(2013)Urvoy, Clerot, F{\'e}raud, and Naamane]{SAVAGE}
Tanguy Urvoy, Fabrice Clerot, Raphael F{\'e}raud, and Sami Naamane.
\newblock Generic exploration and k-armed voting bandits.
\newblock In \emph{International Conference on Machine Learning}, pages 91--99,
  2013.

\bibitem[Vojacek et~al.(2010)Vojacek, Pecakova, et~al.]{RUMegs}
Ondrej Vojacek, Iva Pecakova, et~al.
\newblock Comparison of discrete choice models for economic environmental
  research.
\newblock \emph{Prague Economic Papers}, 19\penalty0 (1):\penalty0 35--53,
  2010.

\bibitem[Xia et~al.(2016)Xia, Qin, Ma, Yu, and Liu]{BMAB16}
Yingce Xia, Tao Qin, Weidong Ma, Nenghai Yu, and Tie-Yan Liu.
\newblock Budgeted multi-armed bandits with multiple plays.
\newblock In \emph{IJCAI}, pages 2210--2216, 2016.

\bibitem[Yue and Joachims(2011)]{BTM}
Yisong Yue and Thorsten Joachims.
\newblock Beat the mean bandit.
\newblock In \emph{Proceedings of the 28th International Conference on Machine
  Learning (ICML-11)}, pages 241--248, 2011.

\bibitem[Yue et~al.(2012)Yue, Broder, Kleinberg, and Joachims]{Yue+12}
Yisong Yue, Josef Broder, Robert Kleinberg, and Thorsten Joachims.
\newblock The k-armed dueling bandits problem.
\newblock \emph{Journal of Computer and System Sciences}, 78\penalty0
  (5):\penalty0 1538--1556, 2012.

\bibitem[Zhou and Tomlin(2017)]{BMAB17}
Datong~P Zhou and Claire~J Tomlin.
\newblock Budget-constrained multi-armed bandits with multiple plays.
\newblock \emph{arXiv preprint arXiv:1711.05928}, 2017.

\bibitem[Zoghi et~al.(2014)Zoghi, Whiteson, Munos, Rijke, et~al.]{Zoghi+14RUCB}
Masrour Zoghi, Shimon Whiteson, Remi Munos, Maarten~de Rijke, et~al.
\newblock Relative upper confidence bound for the k-armed dueling bandit
  problem.
\newblock In \emph{JMLR Workshop and Conference Proceedings}, number~32, pages
  10--18. JMLR, 2014.

\end{thebibliography}
